\documentclass{article}

\usepackage[preprint]{neurips_2025}

\usepackage[utf8]{inputenc} 
\usepackage[T1]{fontenc}    
\usepackage{hyperref}       
\usepackage{url}            
\usepackage{booktabs}       
\usepackage{amsfonts}       
\usepackage{nicefrac}       
\usepackage{microtype}      
\usepackage{xcolor}         

\usepackage{hyperref}
\usepackage{graphicx}
\usepackage{supertabular}
\usepackage{algorithm}
\usepackage{algorithmicx}
\usepackage[noend]{algpseudocode}
\usepackage{subcaption}

\usepackage{amsmath, amsfonts, amssymb, amsthm, amsbsy, amscd, bm, bbm,mathrsfs}   
\usepackage[capitalize,noabbrev]{cleveref}

\usepackage{color}
\usepackage{framed}

\definecolor{shadecolor}{gray}{0.95}
\newenvironment{shadequote}{%
  \begin{center}%
    \begin{minipage}{.95\linewidth}%
      \begin{shaded}%
        \sffamily\slshape}{%
      \end{shaded}
    \end{minipage}%
  \end{center}%
}
\theoremstyle{plain}
\newtheorem{theorem}{Theorem}[section]

\theoremstyle{definition}
\newtheorem{definition}[theorem]{Definition}

\theoremstyle{remark}
\newtheorem{remark}[theorem]{Remark}

\newcommand{\g}{\mathbf{g}}

\newcommand{\G}{\mathbf{G}}

\newcommand{\R}{\mathbb{R}}
\newcommand{\w}{\bm{w}}

\newcommand{\A}{\mathbf{A}}

\renewcommand{\H}{\mathbf{H}}

\newcommand{\Y}{\textcolor{green}{Y}}
\newcommand{\N}{\textcolor{red}{N}}

\title{Adaptive parameter-efficient fine-tuning via Hessian-informed subset selection}

\author{Shiyun Xu\\
\texttt{shiyunxulara@gmail.com} \\
\And
Zhiqi Bu\\
\texttt{woodyx218@gmail.com}}

\begin{document}

\maketitle

\begin{abstract}
Parameter-efficient fine-tuning (PEFT) is a highly effective approach for adapting large pre-trained models to downstream tasks with minimal computational overhead. At the core, PEFT methods freeze most parameters and only trains a small subset (say $<0.1\%$ of total parameters). Notably, different PEFT methods select different subsets, resulting in varying levels of performance. This variation prompts a key question: how to effectively select the most influential subset to train?

We formulate the subset selection as a multi-task problem: maximizing the performance and minimizing the number of trainable parameters. We leverage a series of transformations -- including $\epsilon$-constraint method and second-order Taylor approximation -- to arrive at the classical 0-1 knapsack problem, which we solve through the lens of Pareto optimality. Consequently, we propose AdaPEFT, a Hessian-informed PEFT that adapts to various tasks and models, in which the selected subset empirically transfers across training horizons and model sizes. 
\end{abstract}

\section{Introduction}
Fine-tuning is an important technique in deep learning, which adapts large pre-trained models to new tasks quickly. At high level, there are two classes of fine-tuning methods: (I) fine-tune the entire model (i.e. 100\% of parameters are trainable), or (II) only update a small portion of the model (e.g. 0.1\%) and freeze the majority of parameters. The second class of methods is known as parameter-efficient fine-tuning (PEFT), including examples such as LoRA \cite{hu2022lora}, prompt tuning \cite{lester2021power}, linear probing (or last-layer tuning), LayerNorm tuning \cite{zhaotuning}, and BitFit \cite{zaken2022bitfit}. In contrast to full-model fine-tuning (FMT), PEFT enjoys $\approx 50\%$ speedup and significantly reduced memory cost, e.g. LoRA uses $5\times$ less memory than FMT on LLAMA2-7B \cite{liloftq}. 

On the other hand, the performance of PEFT depends on the choice of methods, or the choice of a subset of parameters. With a proper choice of trainable parameters, PEFT can be as performant as FMT. For example, training RoBERTa-base on 8 datasets in the GLUE benchmark \cite{wangglue}, FMT gets an average accuracy 86.4\%, LoRA gets 87.2\%, and BitFit gets 85.2\%(see Table 2 in \cite{hu2022lora}); training GPT3 on WikiSQL \cite{zhongSeq2SQL2017}, FMT gets an accuracy 73.8\%, LoRA gets 73.4\%, and BitFit gets 71.3\% (see Table 4 in \cite{hu2022lora}). However, the success of these PEFT methods may not be reproduced in some tasks. For example, LoRA significantly underperforms FMT on CoLA and MRPC datasets  (see Table 1 of \cite{wang2024lora,wang2024lora2}); on coding and mathematics domains, \cite{biderman2024lora} shows that "LoRA substantially underperforms full finetuning" in both instruction tuning and continued learning ($\approx 20B$ tokens).

As a consequence, 
we study the following question:
\begin{shadequote}
Q: How to adaptively select parameter groups (or active subset of parameters) so that the model achieves high utility while only has a small portion of trainable parameters?
\end{shadequote}

To answer this question, we build on top of existing PEFT methods and formulate a \textbf{multi-task problem} on the active set of parameter groups --- $\A$.
\begin{align}
\min_{\A} \left(L_\A:=\min_{\w_{(k)}\in \A} L(\w)\right) \quad\&\quad \min_\A\left(\frac{|\A|}{|\w|}:=\frac{\sum_k \mathbb{I}(\w_{(k)}\in\A)\cdot|\w_{(k)}|}{\sum_k |\w_{(k)}|} \right)
\label{eq:main question0}
\end{align} 

Here $\w:=[\w_{(1)}, \w_{(2)}, ..., \w_{(K)}]$ is the set of parameter groups of a model, where the grouping is inspired by existing PEFT methods. For instance, as shown in \Cref{tab:peft summary}, we see that each PEFT method corresponds to some fixed subset $\A\subseteq \w$.

\begin{table}[!htb]
    \centering
    \caption{Summary of PEFT methods (row) and corresponding parameter groups (column). Here `lora\_A/lora\_B' are low-rank matrices, `head' is the last linear layer, `norm' is layer normalization, and `bias' is bias terms. Y/N indicate whether a parameter group is trainable.}
    \begin{tabular}{|c|c|c|c|c|c|}
    \hline    &lora\_A&lora\_B&head&norm&bias  \\\hline
LoRA & Y& Y&N&N&N\\ 
LoRA-FA &N& Y&N&N&N\\ 
BitFit&N&N&N&N&Y\\
Linear probing&N&N&Y&N&N\\
LayerNorm &N&N&N&Y&N\\\hline
AdaPEFT (ours)&\multicolumn{5}{c|}{adaptive}\\   \hline
   \end{tabular}
    \label{tab:peft summary}
\end{table}
When $\A$ is determined by a PEFT method, we solve the inner problem in \eqref{eq:main question0} over $\w$:
$$\min_{\w_{(k)}\in\A} L(\w)\equiv\min_{\w_{(k)}\in\A} L(\w_{(1)},...,\w_{(K)}).$$
Note when $\A=\w$, we recover the FMT problem.

In contrast to applying a fixed $\A$, we enable the adaptivity of $\A$ in \eqref{eq:main question0}, where we have two tasks: we minimize the loss and minimize the number of trainable parameters. These two tasks may conflict with each other. For instance, the minimum number of trainable parameters is 0 but then the model is not trainable at all, hence the loss is not minimized. 

To resolve the conflict, we solve our multi-task problem under Pareto optimality. 
\begin{definition} (Pareto optimality for PEFT). For two active sets $\A_1$ and $\A_2$, if $L_{\A_1} \leq L_{\A_2}$ and $|\A_1|\leq |\A_2|$ with at least one inequality being strict, then $\A_1$ dominates $\A_2$. Furthermore, an active set is Pareto optimal if it is not dominated by any other sets.
\end{definition}



\paragraph{Contribution.}
\begin{itemize}
    \item We formulate a multi-task problem to adaptively select the active set of parameter groups for PEFT, through the lens of Pareto optimality. 
    \item We transform our multi-task problem (optimized on subsets and weights) to a classical single-task problem (optimized on binary variables) that is known as 0-1 knapsack problem. Specifically, our transformation leverages a series of methods including $\epsilon$-constraint method, gradient descent, and Taylor approximation.
    \item We propose efficient algorithms to compute the Hessian-informed loss reduction and to solve the 0-1 knapsack problem.
    \item We observe consistent patterns indicating that influential parameter groups can be discovered early in training, and such patterns can transfer across model sizes. These patterns are visualized on various tasks and models.
\end{itemize}

\paragraph{Related work.}
We briefly discuss some related work and refer to \Cref{app:related work} for details. Broadly speaking, this work connects to all subset-based PEFT methods, as our problem can be formulated over any collection of such methods. To solve our 0-1 knapsack problem, we draw on technique from \cite{bu2024automatic} to compute Hessian-informed loss reduction, and employ a greedy approximation algorithm \cite{martello1990knapsack} to estimate the Pareto frontier.

\newpage
\section{Problem formulation}
    \begin{table}[!htb]
    \centering
\caption{Roadmap of transformations from multi-task optimization to 0-1 knapsack problem.}
    \begin{tabular}{c|c|c|c}
       reference& problems &variables & transformation \\\hline
        \eqref{eq:main question0} & meta-minimization \& minimization & (subset, $\w$) &-- --\\
        \eqref{eq:eps-constraint method} & constrained meta-minimization & (subset, $\w$) &$\epsilon$-constraint\\
        \eqref{eq:main question3} & constrained meta-minimization & (subset, $\eta$) &gradient descent\\        
        \eqref{eq:main question4} & constrained maximization & subset &Taylor approximation\\     
        \eqref{eq:binary question} & constrained maximization & binary &knapsack problem\\    
        \end{tabular}
    \label{tab:my_label}
\end{table}
\subsection{Notations}
\label{sec:notation}
We denote $\w\in\R^D$ as all model parameters, while $\w_t$ represents the iteration $t$ and $\w_{(k)}$ represents the $k$-th parameter group. 
The same notation follows for other variables including the mini-batch gradient $\g\in\R^D$. We denote the loss as $L(\w)$, its first-order derivative as $\G_{(k)}:=\frac{\partial L}{\partial \w_{(k)}}$ and its second-order derivative as $\H_{(k)}:=\frac{\partial^2 L}{\partial\w_{(k)}^2}$. We omit $t$ when it is obvious from the context. We use $|\cdot|$ to count the number of parameters.

\subsection{Multi-task optimization and Pareto optimality}
The $\epsilon$-constraint method is a scalarization technique to solve multi-task problem. It chooses one task to optimize and converts the remaining tasks into constraints. Thus, \eqref{eq:main question0} leads to
\begin{align}
    \min_{\A} L_{\A} \quad\text{s.t.}\quad|\A|/|\w|\leq \epsilon
    \label{eq:eps-constraint method}
\end{align}

However, we emphasize that a solution to \eqref{eq:eps-constraint method} is not necessarily Pareto optimal in terms of \eqref{eq:main question0} (shown in \Cref{thm:not Pareto}), unless we restrict ourselves to an unrealistic case where \eqref{eq:eps-constraint method} has a unique solution. 

\begin{theorem}
For any $\epsilon\geq 0$, a solution to \eqref{eq:eps-constraint method} may not be Pareto optimal of \eqref{eq:main question0}. Nevertheless, if \eqref{eq:eps-constraint method} has only one solution, then the solution is Pareto optimal.
\label{thm:not Pareto}
\end{theorem}

\begin{proof}
Suppose $\A^\star$ is a solution to \eqref{eq:eps-constraint method} but not Pareto optimal, then $\exists \A^\prime$ such that, 
  \begin{align}
      L_{\A^\prime}\leq L_{\A^\star} \text{ and } |\A^\prime|\leq |\A^\star|
  \end{align}
  with at least one strict inequality. $\A^\star$ being the minimizer of \eqref{eq:eps-constraint method} gives $L_{\A^\star}\leq L_{\A^\prime}$. Hence, $L_{\A^\prime}= L_{\A^\star}$. Then the strict inequality only happens at $|\A^\prime|< |\A^\star|$. However, $|\A^\prime|< |\A^\star|\leq \epsilon$ guarantees that $\A^\prime$ is also a solution to \eqref{eq:eps-constraint method}, which contradicts that $\A^\star$ is the unique solution. 
\end{proof}

To guarantee the Pareto optimality when \eqref{eq:eps-constraint method} has more than one solutions, we propose a \textbf{refined solution} in two steps: (I) exhaustively find all solutions of \eqref{eq:eps-constraint method}, denoted by a set $\text{argmin}_{\A:|\A|/|\w|\leq\epsilon} L_\A$; (II) select one solution with the smallest number of trainable parameters. Mathematically, we select
\begin{align}
\A^\star(\epsilon)=\text{argmin}_\A \left\{|\A|: \A\in\text{argmin}_{\A:|\A|/|\w|\leq\epsilon} L_\A\right\}
\label{eq:A pareto}
\end{align}
\begin{theorem}
For any $\epsilon\geq 0$, the refined solution to \eqref{eq:eps-constraint method} (i.e. \eqref{eq:A pareto}) is always Pareto optimal of \eqref{eq:main question0}.
\label{thm:yes Pareto}
\end{theorem}
\begin{proof}
Similar to the proof of \Cref{thm:not Pareto}, if $\exists\A^\prime$ that dominates $\A^\star$, then $L_{\A^\prime}=L_{\A^\star}$ and $|\A^\prime|<|\A^\star|$. Hence $\A^\prime$ belongs to the solution set, which contradicts that $\A^\star$ has the smallest $|\A|$ among all solutions.
\end{proof}


\begin{remark}
\Cref{thm:not Pareto} and \Cref{thm:yes Pareto} hold without requiring convexity. Note each $\epsilon$ yields one $\A^\star(\epsilon)$ on the Pareto frontier, hence one can sweep through $0\leq\epsilon\leq 1$ to get different Pareto optimal solutions. 
\end{remark}

\subsection{Gradient descent for inner minimization}
In this section, we provide a closed-form approximation for the inner minimization problem of \eqref{eq:eps-constraint method}, i.e. $L_\A:=\min_{\w_{(k)}\in \A} L(\w)$, from an iterative perspective. Algorithmically, the inner minimization is solved iteratively via gradient descent, such as SGD and Adam. For example, for $1\leq t\leq T$,
\begin{align}
\begin{split}
\w_{(k),t+1}=\w_{(k),t}-\eta I_{(k)}\g_{(k),t}:=\w_{(k),t}-\eta\mathbb{I}(\w_{(k)}\in\A)\g_{(k),t}
\end{split}
\label{eq:peft gd}
\end{align}
in which the binary mask $\mathbb{I}(\w_{(k)}\in\A)$ assigns 0 to frozen groups and 1 to active groups.
Note \eqref{eq:peft gd} recovers FMT: $\w_{t+1}=\w_t-\eta\g_t$ since $\mathbb{I}(\w_{(k)}\in\w)\equiv 1$.

As a consequence, the gradient descent translates the inner minimization $\min_{\w_{(k)}\in\A}L(\w)$ to another minimization over $\min_{\eta\in\R}L(\w_T)$. This changes the high-dimensional parameter optimization to a one-dimensional hyperparameter optimization:
\begin{align}
L_\A&\approx 
\min_{\eta} L(\w_{T})  \text{ s.t. }\eqref{eq:peft gd}
\label{eq:main question2}
\end{align}
Furthermore, the optimization problem \( \min_{\w} L(\w) \) is unconstrained, whereas the optimization $\min_{\eta \text{ s.t. } \eqref{eq:peft gd}} L(\w_T(\eta))$ is restricted to the gradient descent path \( \{\w_t\}_t \) governed by the update rule \eqref{eq:peft gd}. We know $\min_{\eta \text{ s.t. } \eqref{eq:peft gd}} L(\w_T(\eta))\geq\min_{\w} L(\w)$, with the equality holding if and only if a global minimizer of \( L \) lies on the path \( \{\w_t\}_t \).

To enjoy a closed-form approximation, we study the inner minimization from a local perspective (i.e. one iteration at a time), \eqref{eq:main question2} is equivalent to
\begin{align}
L_\A\approx&\sum_{t=1}^{T-1}\min_{\eta} [L(\w_{t+1};\eta)-L(\w_{t})] \text{ s.t. }\eqref{eq:peft gd}
\label{eq:main question3}
\end{align}
up to a constant $L(\w_0)$. At each iteration,  by only updating one parameter group (say $\A=\{\w_{(k)}\}$) and freezing all the others, the loss reduction from this $\w_{(k)}$ is
\begin{align}
&L(\w_{t+1};\eta)-L(\w_t)
\approx\Delta L_{(k),t}(\eta):= -\eta\G_{(k),t}^\top \g_{(k),t}+\frac{\eta^2}{2}
\g_{(k),t}^\top\H_{(k),t}\g_{(k),t}
\label{eq:lr parabola}
\end{align}
We note that the second-order Taylor approximation is reasonably accurate, since the approximation error is $o(\eta^2)$ and hence negligible in practice (c.f. Figure 2 in \cite{bu2024automatic}).

Therefore, if $\g_{(k)}^\top\H_{(k)}\g_{(k)}>0$, updating the $k$-th parameter group can minimize \eqref{eq:lr parabola} to $-\frac{\left(\G_{(k)}^\top \g_{(k)}\right)^2}{\g_{(k)}^\top\H_{(k)}\g_{(k)}}$ under $\eta=\frac{\G_{(k)}^\top \g_{(k)}}{\g_{(k)}^\top\H_{(k)}\g_{(k)}}$. Extending to multiple parameter groups, we write the total loss reduction as
\begin{align}
L_\A\approx\sum_{k,t} \Delta L_{(k),t}=-\sum_k\mathbb{I}(\w_{(k)}\in\A)\cdot\sum_t\frac{\left(\G_{(k),t}^\top \g_{(k),t}\right)^2}{\g_{(k),t}^\top\H_{(k),t}\g_{(k),t}}
\label{eq:main question5}
\end{align}
which explicitly replaces the inner minimization problem of \eqref{eq:main question3}.

\subsection{Knapsack problem}
All in all, we transform the meta-minimization in \eqref{eq:eps-constraint method} to a subset maximization problem in \eqref{eq:main question4},
\begin{align}
\min_{\A} \left(\min_{\w_{(k)}\in \A} L(\w)\right)
&\approx\max_{\A}\sum_k\mathbb{I}(\w_{(k)}\in\A)\cdot\sum_t\frac{\left(\G_{(k),t}^\top \g_{(k),t}\right)^2}{\g_{(k),t}^\top\H_{(k),t}\g_{(k),t}}
\text{ s.t. } |\A|/|\w|\leq \epsilon
\label{eq:main question4}
\end{align}
Finally, given that $\A$ is only reflected in the binary variable $\mathbb{I}(\w_{(k)}\in\A)$, we can further simplify the subset maximization in \eqref{eq:main question4} to a binary maximization problem:
\begin{align}
\max_{I_k\in\{0,1\}}\sum_k I_k\cdot\sum_t\frac{\left(\G_{(k),t}^\top \g_{(k),t}\right)^2}{\g_{(k),t}^\top\H_{(k),t}\g_{(k),t}} \text{ s.t. } \frac{\sum_k I_k|\w_{(k)}|}{\sum_k |\w_{(k)}|}\leq \epsilon
\label{eq:binary question}
\end{align}

Importantly, \eqref{eq:binary question} is essentially the 0-1 knapsack problem --- a classical NP-complete combinatorial problem.

\begin{shadequote}
Given a set of items, each with a weight $W_k$ and a value $V_k$, the knapsack problem determines which items to include so that the total value is maximized within the total weight limit $\mathcal{W}$:
\begin{align}
    \max_{I_k\in \{0, 1\}} \sum_{k} V_k \times I_k \text{ s.t. } \sum_{k} W_k\times I_k \leq \mathcal{W}
\end{align}
where the $k$-th item is included if and only if $I_k=1$.
\end{shadequote}
In our analysis of PEFT, we view each parameter group as an item, with parameter count as weight $W_k:=|\w_{(k)}|$ and loss reduction as value $V_k:=\sum_t\frac{\left(\G_{(k),t}^\top \g_{(k),t}\right)^2}{\g_{(k),t}^\top\H_{(k),t}\g_{(k),t}}$, under the limit $\mathcal{W}:=\epsilon\sum_k|\w_{(k)}|$.

With hindsight, \eqref{eq:binary question} is equivalent to applying $\epsilon$-constraint method to the following multi-task optimization,
\begin{align}
\max_{\A}\sum_k \mathbb{I}(\w_{(k)}\in\A)\cdot\sum_t\frac{\left(\G_{(k),t}^\top \g_{(k),t}\right)^2}{\g_{(k),t}^\top\H_{(k),t}\g_{(k),t}} \quad\&\quad \min_{\A}\frac{\sum_k \mathbb{I}(\w_{(k)}\in\A)\cdot|\w_{(k)}|}{\sum_k |\w_{(k)}|}.
\label{eq:binary question MTO}
\end{align}

\section{Algorithms}
We discuss two classes of algorithms -- one to efficiently compute the objectives of the knapsack problem, and the other to actually solve the knapsack problem.

\subsection{Efficiently computing loss reduction without back-propagation}
\label{sec:efficiency}
We propose \Cref{alg:autoSGD} to efficiently compute the loss reduction, which requires the knowledge of $\G_{(k)}^\top\g_{(k)}$ and $\g_{(k)}^\top\H_{(k)}\g_{(k)}$. We use two techniques, (I) quadratic curve fitting and (II) lazy updating.

\begin{algorithm}[!htb]
\caption{Hessian-informed selection of parameter groups (iteration $t$)}
\begin{algorithmic}[1]
    \State Compute loss $L_0 = L(\w_t)$ by forward pass
    \State Compute gradient $\g:=\{\g_{(k)}\}_k$ on $L_0$ by SGD, AdamW, etc.
    \If{$t \text{ mod } 4K==0$}
    \For{$k\in 1,...,K$}
    \For{$\hat\eta_j\in\{-2,-1,1,2\}\cdot\eta_t$}
    \State Compute $L_j = L\left(\w_t - \hat\eta_j\bm{e}_k\odot\g\right)$ by forward pass
    \EndFor
    \State Fit a quadratic curve from $\{\hat\eta_j\} \rightarrow \{L_j - L_0\}$
    \State Derive $\G_{(k),t}^\top \g_{(k),t}$ and $\g_{(k),t}^\top \H_{(k),t} \g_{(k),t}$ in \eqref{eq:lr parabola}
        \If{$\G_{(k),t}^\top \g_{(k),t}>0, \g_{(k),t}^\top \H_{(k),t} \g_{(k),t}>0, \text{R2 score}>0.99$}
\State Accumulate the loss reduction
$V_k = V_k+ \frac{|\G_{(k),t}^\top\g_{(k),t}|^2}{\g_{(k),t}^\top\H_{(k),t}\g_{(k),t}}$
        \State Derive optimal learning rates $\eta_{(k)}^*=\frac{\G_{(k),t}^\top\g_{(k),t}}{\g_{(k),t}^\top\H_{(k),t}\g_{(k),t}}$ 
        \EndIf
\EndFor
\EndIf

    \State Update the parameters: 
    $\w_{(k),t+1} = \w_{(k),t} - \eta_{(k)}^*\g_{(k),t}$
\end{algorithmic}
\label{alg:autoSGD}
\end{algorithm}
\paragraph{Quadratic curve fitting.}
We fit the quadratic function \eqref{eq:lr parabola} by using multiple forward passes at different learning rates to get different $\Delta L_{(k)}(\eta)$. Next, we directly find the two scalars $\G_{(k)}^\top\g_{(k)}$ and $\g_{(k)}^\top\H_{(k)}\g_{(k)}$ via a finite-sum problem:
$$\g_{(k)}^\top\H_{(k)}\g_{(k)},\G_{(k)}^\top\g_{(k)}\approx\text{argmin}_{A,b} \sum_j \left(\Delta L_{(k)}(\hat\eta_j)+\hat\eta_j b-\frac{\hat\eta_j^2}{2}A\right)^2$$
where $\hat\eta_j\in \{-2\eta,-\eta,0,\eta,2\eta\}$. This back-propagation-free approach from \cite{bu2024automatic} has minimal memory overhead, as it does not instantiate 
$\G_{(k)}$ or $\H_{(k)}$. 

\paragraph{Lazy updating.}
We need $4K+1$ forward passes, hence incurring $O(K)$ training time overhead if implemented naively. This overhead can be reduced significantly if the loss reduction is computed infrequently, following \cite{bu2024automatic}. In practice, we update every $O(K)$ iterations so that the overhead is $O(1)$ and roughly negligible.

\subsection{Solving 0-1 knapsack problem}
\paragraph{Exhaustive search}
To obtain the Pareto optimal solution of \eqref{eq:binary question MTO}, we need to exhaustively search all $2^K$ subsets and compute the value $\sum_k\mathbb{I}(\w_{(k)}\in\A)\cdot V_k$ and the weight $\sum_k\mathbb{I}(\w_{(k)}\in\A)\cdot W_k$ for each subset. Then we find all subsets with $|\A|/|\w|\leq \epsilon$ and select the subset with the largest value. This algorithm guarantees to find a Pareto optimal minimizer by \Cref{thm:yes Pareto}, but it has an exponential complexity in terms of $K$.

\paragraph{Approximate solutions}
In practice, we turn to approximate solutions like the greedy approximation \cite{martello1990knapsack} to solve \eqref{eq:binary question} instead of directly \eqref{eq:binary question MTO}. We define the value-to-weight ratio as \textit{Per-Parameter Influence} (PPI) and its accumulation where
\begin{align}
\text{PPI}_{k}(t)= \frac{(\G_{(k),t}^\top\g_{(k),t})^2}{\g_{(k),t}^\top\H_{(k),t}\g_{(k),t}\cdot |\w_{(k)}|},
\quad \text{APPI}_{k}(\tau)= \sum_{t=1}^\tau \text{PPI}_{k}(t)
\label{eq:PPI}
\end{align}
is computed by \Cref{alg:autoSGD}.

Next, we (I) sort $\{\w_{(k)}\}_k$ by $\text{APPI}_{k}$ in descending order; (II) output $K$ subsets $\{\A_1,...\A_K\}$, with the $k$-th subset containing the first $k$ parameter groups. 

As a result, we view $\A_k$ as the solution to \eqref{eq:binary question} for any $\epsilon\in \left[\frac{|\A_k|}{|\w|},\frac{|\A_{k+1}|}{|\w|}\right)$\footnote{Alternatively, we may use dynamic programming or meet-in-the-middle algorithm algorithm to solve \eqref{eq:binary question} exactly. However, each of these algorithms gives one minimizer of \eqref{eq:binary question}, which is not guaranteed to be Pareto optimal in terms of \eqref{eq:binary question MTO} by \Cref{thm:not Pareto}, if the problem has multiple minimizers. Furthermore, to estimate the Pareto frontier, these algorithms need to be applied multiple times, for different $\epsilon$'s, whereas the greed approximation is applied once as it is $\epsilon$-independent.}, and as an approximately Pareto optimal solution to \eqref{eq:binary question MTO}. Rigorously speaking, by \Cref{thm:not Pareto}, $\A_k$ is not guaranteed to be Pareto optimal, yet it provides empirical guidance for adaptive PEFT design later in \Cref{sec:adapeft}.

\section{Visualization of parameter group influence}
In this section, we visualize PPI and APPI of \eqref{eq:PPI} in training. We experiment with multiple tasks (image classification, natural language understanding, and generation), model architectures (Vision Transformer or ViT \cite{dosovitskiy2020image}, T5 \cite{raffel2020exploring}, RoBERTa \cite{liu2019roberta}, GPT \cite{radford2019language}) and model sizes ($\sim 0.1 - 1B$). We partition models into parameter groups that are fine-tuned by LoRA (\cite{hu2022lora}; with module names \textit{lora\_A} and \textit{lora\_B}), BitFit (\cite{zaken2022bitfit}; with name \textit{bias}), linear probing (with name \textit{head}), LayerNorm tuning (\cite{zhaotuning}; with name \textit{norm}), and embedding layer tuning (with name \textit{embed}).

We generate two types of figures: (I) heatmap of PPI for different parameter groups at different iterations, i.e. $\text{PPI}_{k}(t)$, where the lighter color indicates a stronger influence; (II) line plot of accumulative PPI, i.e. $\text{APPI}_{k}(\tau)\approx L(\w_0)-L(\w_{\tau})$. We leave more details in \Cref{app:visual}.

We consistently observe that some parameter groups are highly influential, with $\approx 10^4\times$ higher PPI than the majority of model parameters. This observation supports the effectiveness of PEFT, only if we actually select the influential parameters.

\subsection{Discrepancy across models and tasks}
In \Cref{fig:heatmap all} and \Cref{fig:heatmap model size}, we observe that the influence of parameter groups varies significantly across models and tasks.  In summary, none of the methods is performant in all scenarios, which motivates an adaptive PEFT design in \Cref{sec:adapeft}.

\paragraph{Varying models and tasks}
In \Cref{fig:heatmap all}, it is clear that PPI patterns are highly dependent on models and tasks, and we can leverage such patterns to shed light on the effectiveness of \textit{new PEFT methods}. In fact, it is no surprise that a combination of multiple PEFT methods can give very strong performance. For example, the LoRA library \cite{hu2022lora} states that \textit{``training bias vectors in tandem with LoRA might be a cost-efficient way to squeeze out extra task performance"}.

\begin{figure}[!htb]
\centering
\begin{subfigure}{0.3\linewidth}
\includegraphics[width=\linewidth]{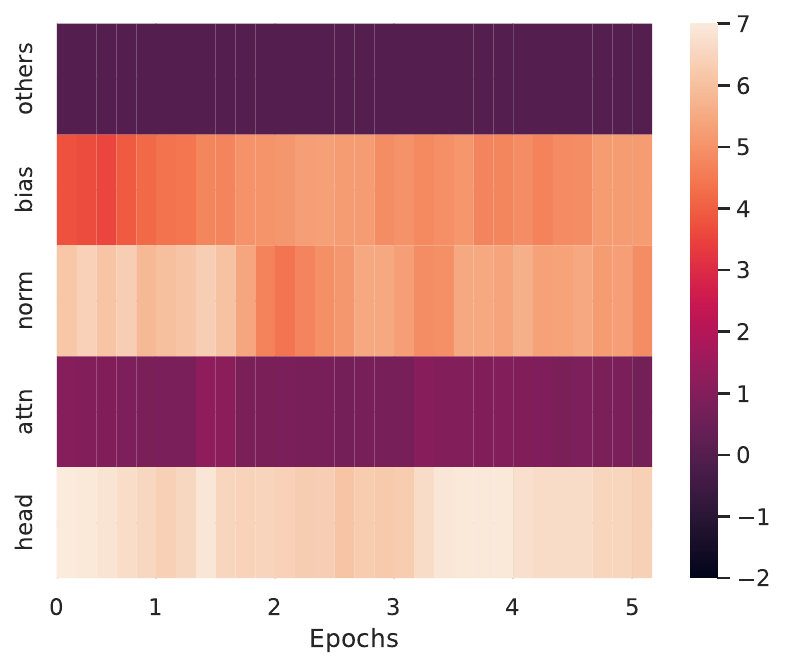}
\caption{CIFAR100, ViT-base, fine-tuning}
\end{subfigure}
\hfill
\begin{subfigure}{0.3\linewidth}
\includegraphics[width=\linewidth]{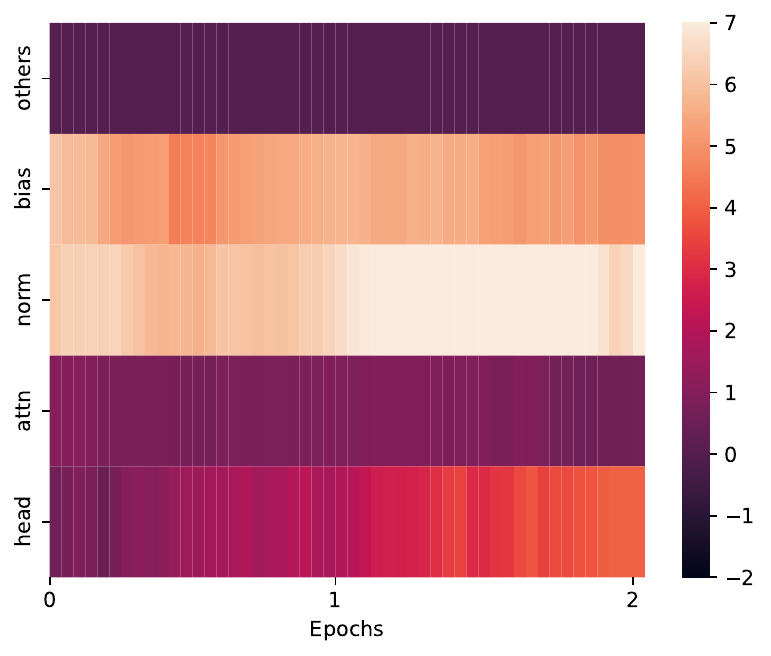}
\caption{ImageNet, ViT-base, fine-tuning\quad\quad\quad\quad\quad\quad}
\end{subfigure}
\hfill
\begin{subfigure}{0.3\linewidth}
\includegraphics[width=\linewidth]{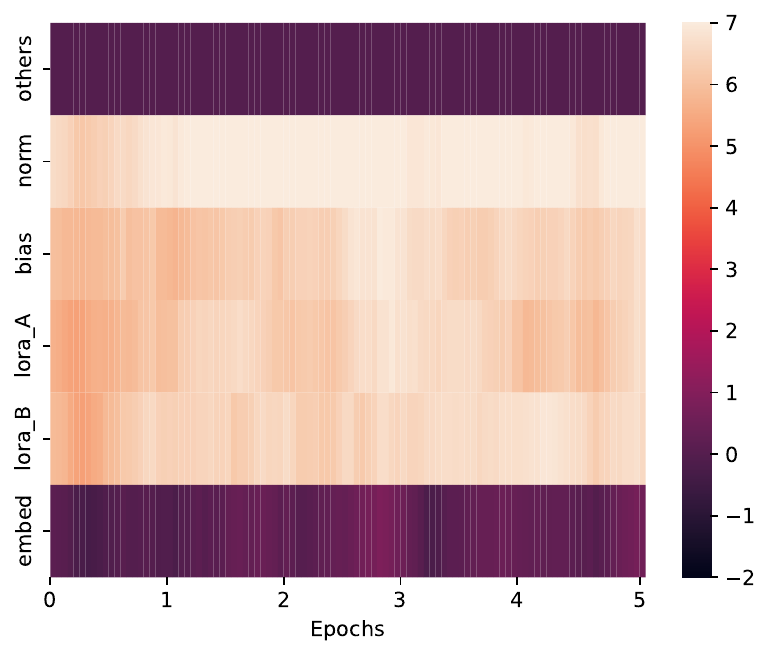}
\caption{E2E, GPT2, fine-tuning\quad\quad\quad\quad\quad}
\end{subfigure}
\begin{subfigure}{0.3\linewidth}
\includegraphics[width=\linewidth]{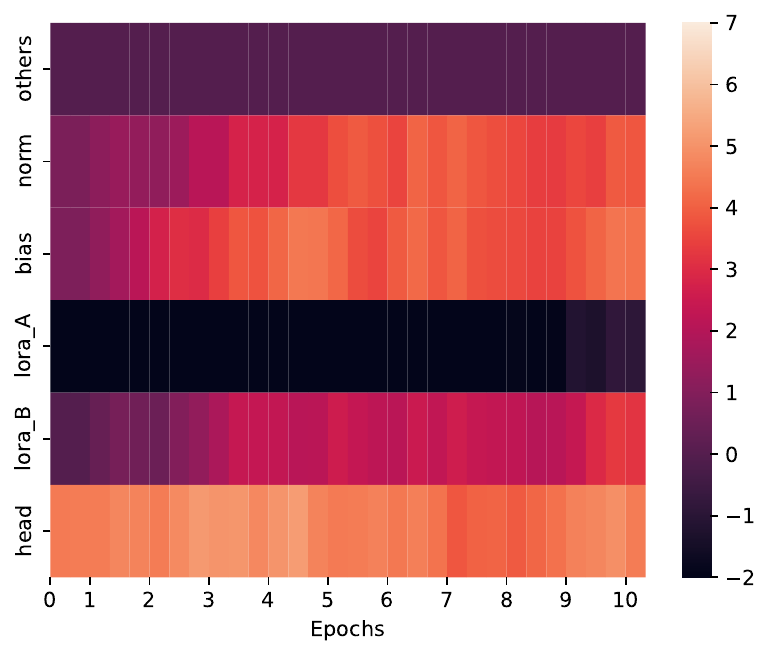}
\caption{MRPC, RoBERTa-base, fine-tuning}
\end{subfigure}
\hfill
\begin{subfigure}{0.3\linewidth}
\includegraphics[width=\linewidth]{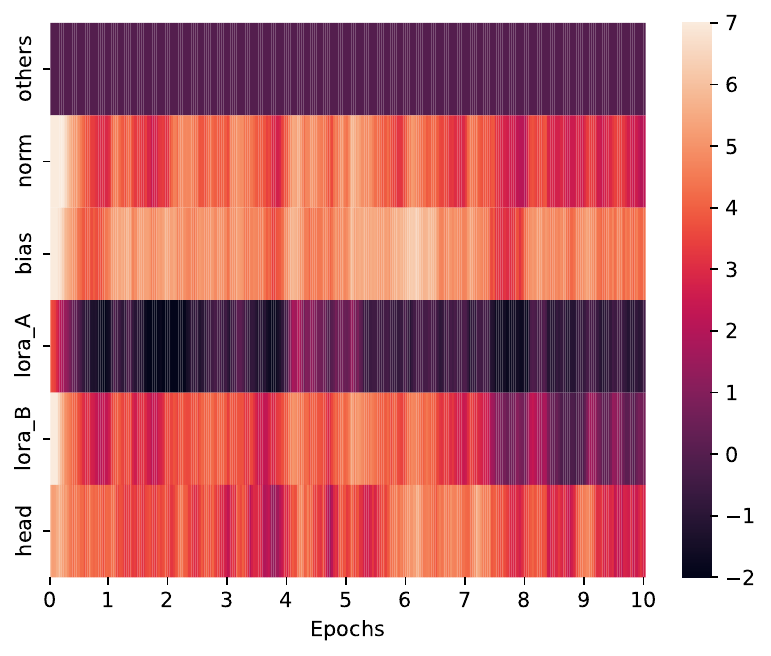}
\caption{SST2, RoBERTa-base, fine-tuning}
\end{subfigure}
\hfill
\begin{subfigure}{0.3\linewidth}
\includegraphics[width=\linewidth]{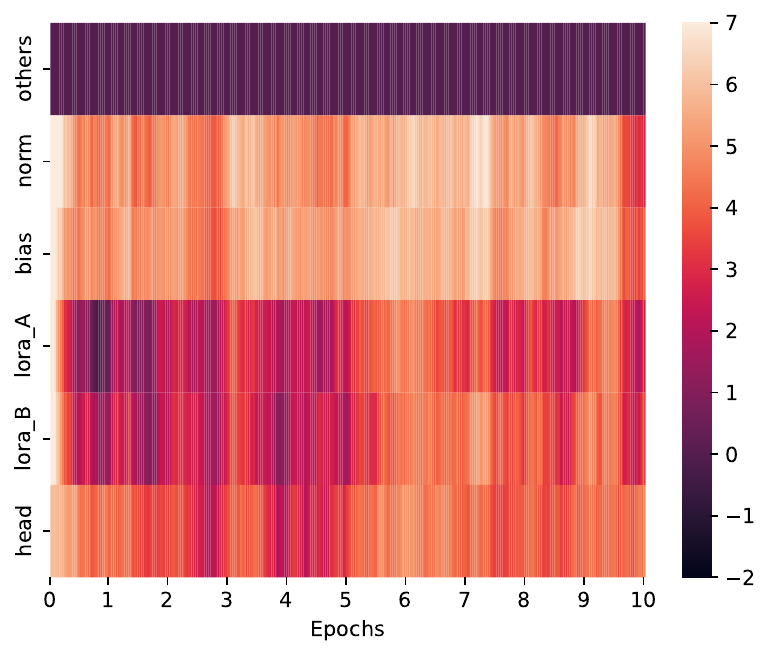}
\caption{SST2, RoBERTa-base, pre-training}
\end{subfigure}
\caption{Heatmap of PPI for multiple parameter groups in log-scale.}
\label{fig:heatmap all}
\end{figure}

For examples in image classification, on CIFAR100, ViT can be effectively trained with a combination of LayerNorm tuning and linear probing, whereas on ImageNet, only LayerNorm tuning may be sufficient. For examples in language modeling,  GPT2 on E2E can benifit from a combination of LoRA, BitFit and LayerNorm tuning; RoBERTa-base can be effectively fine-tuned on MRPC and SST2 by linear probing and BitFit, but it may freeze LoRA A matrix (like in LoRA-FA \cite{zhang2023lora}). In what follows, we demonstrate that PPI can be different by varying only the model or only the task (e.g. dataset or training stage).

\paragraph{Varying tasks with fixed model}
A closer look at \Cref{fig:heatmap all}(e)(f), both training RoBERTa model on SST2 dataset, reveals that PPI can be different at different training stages: lora\_A matrix is less influential in fine-tuning than in pre-training. Additionally, comparing \Cref{fig:heatmap all}(a)(b) or comparing \Cref{fig:heatmap all}(d)(e) reveals that PPI can be different on the same model when the datasets change.

\paragraph{Varying models with fixed task}
In \Cref{fig:heatmap model size}, we compare PPI on T5 and RoBERTa models on the same task. It is clear that different model architectures can have different patterns.

\begin{figure}[!htb]
\centering
\includegraphics[width=0.24\linewidth]{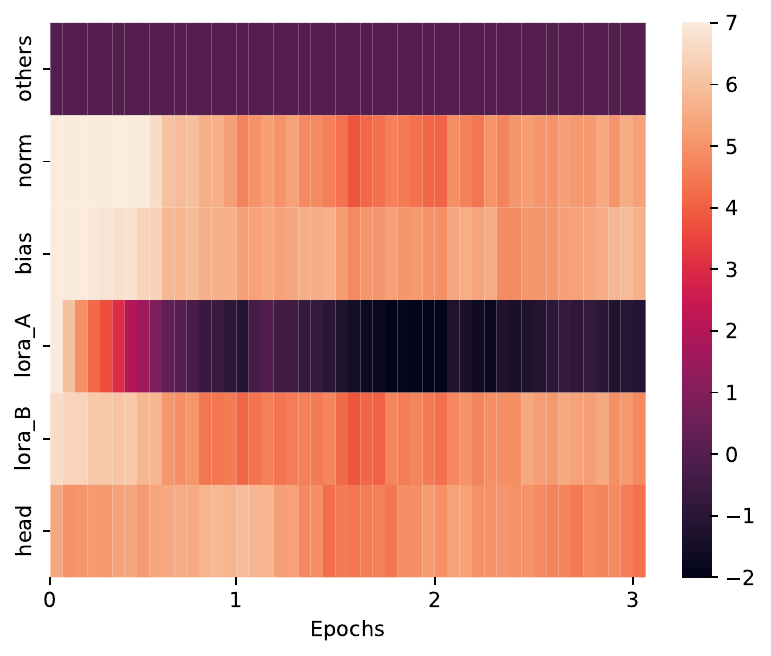}
\includegraphics[width=0.24\linewidth]{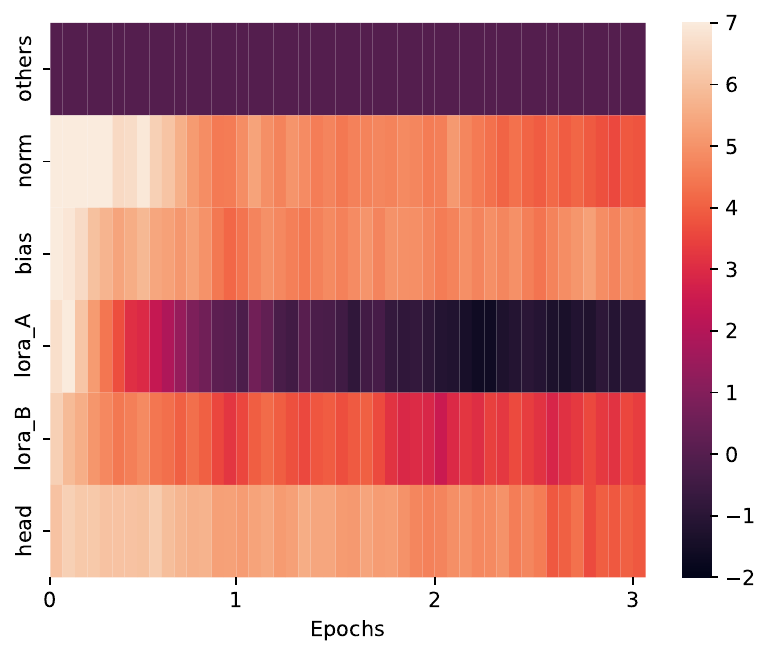}
\includegraphics[width=0.24\linewidth]{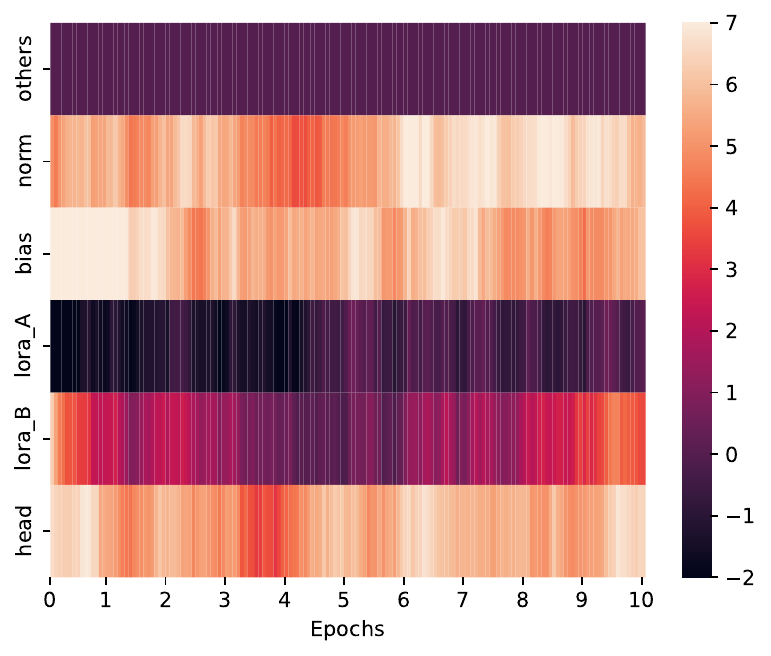}
\includegraphics[width=0.24\linewidth]{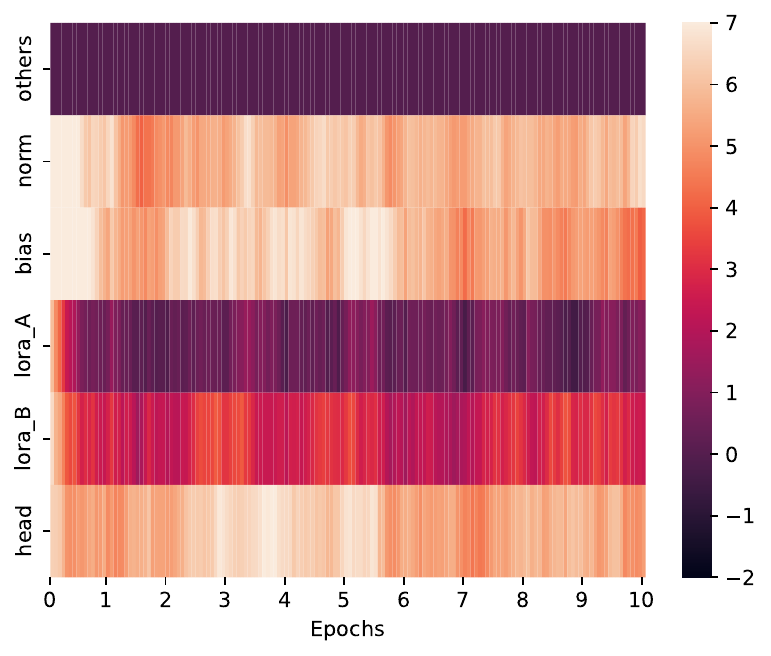}
\caption{Heatmap of PPI on CoLA in log-scale. Left to right: RoBERTa-base, RoBERTa-large, T5-small, and T5-base.}
\label{fig:heatmap model size}
\end{figure}

\subsection{Consistency across training iterations}
We empirically observe a consistent pattern of PPI and APPI across iterations, shortly after the initialization of models. See \Cref{fig:heatmap all} and \Cref{fig:heatmap model size} for PPI, and \Cref{fig:appi} for APPI. Such consistency motivates an efficient strategy to design PEFT: train FMT for some iterations (say 10\% of the full run), determine PEFT based on APPI, then launch the full run with PEFT.

\begin{figure}[!htb]
    \centering
\includegraphics[width=0.19\linewidth]{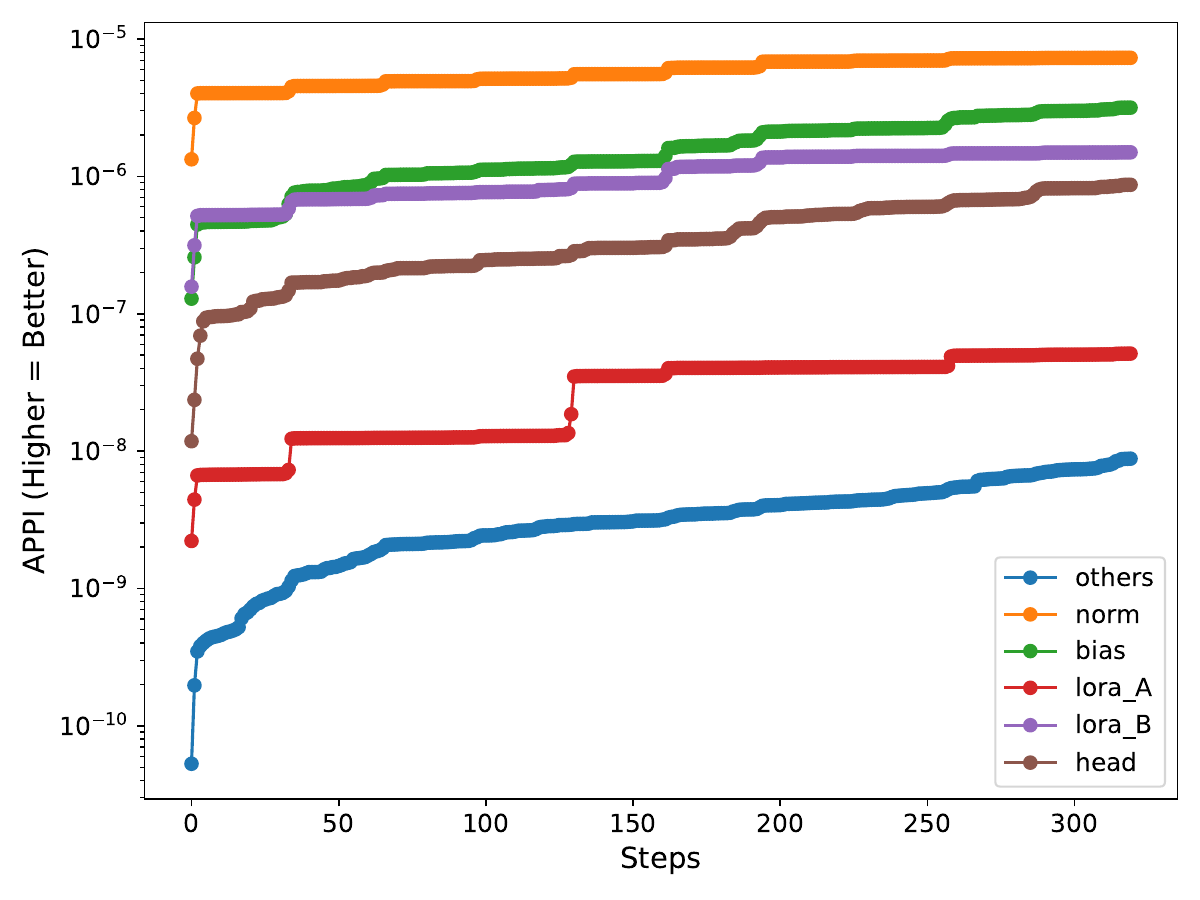}
    \includegraphics[width=0.19\linewidth]{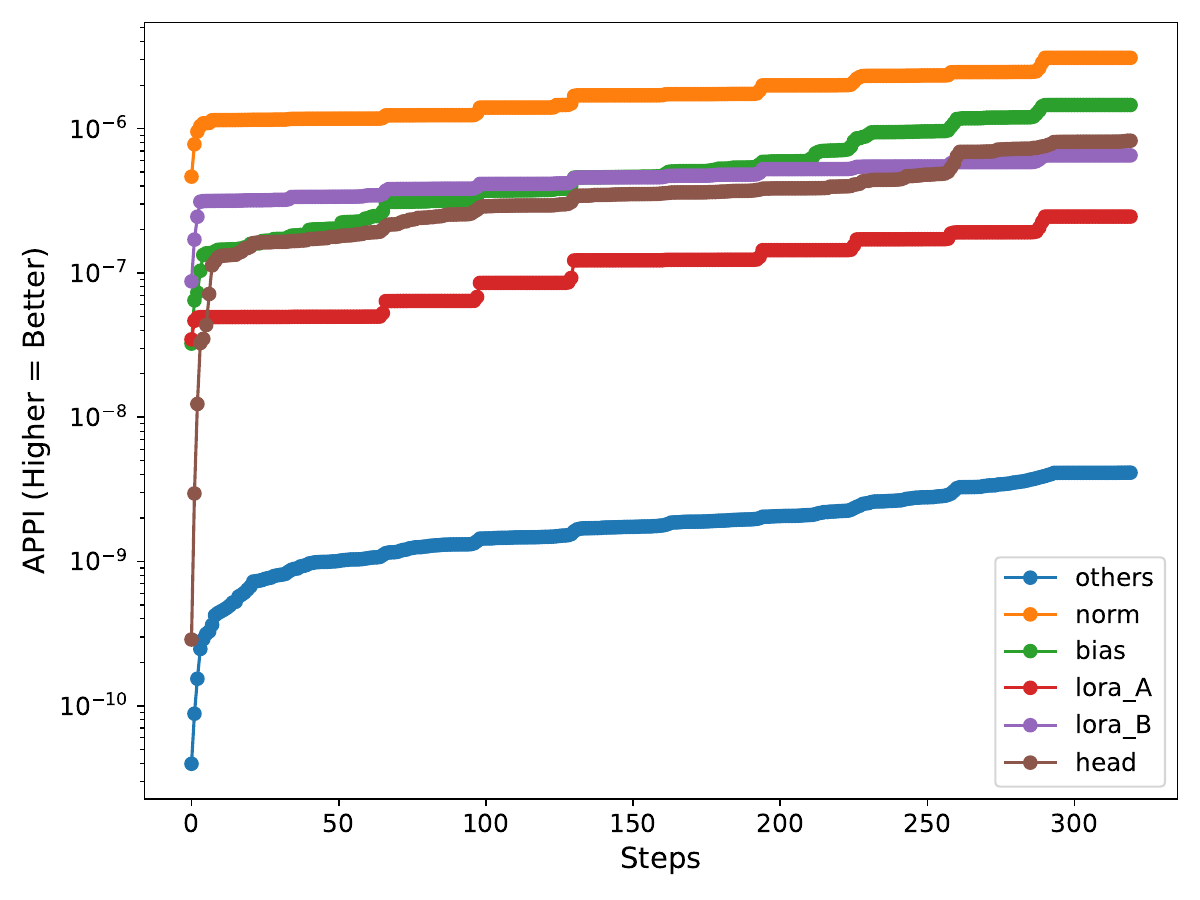}
\includegraphics[width=0.19\linewidth]{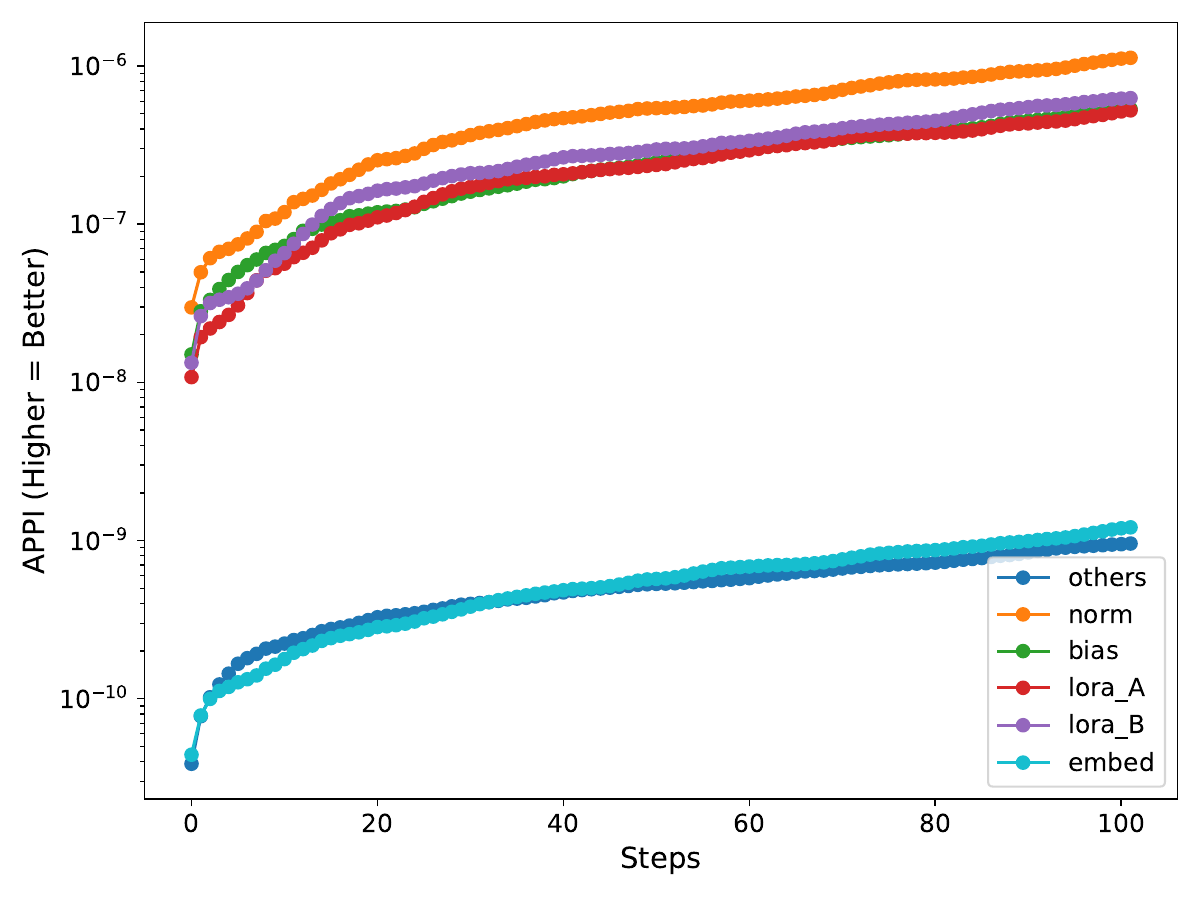}
    \includegraphics[width=0.19\linewidth]{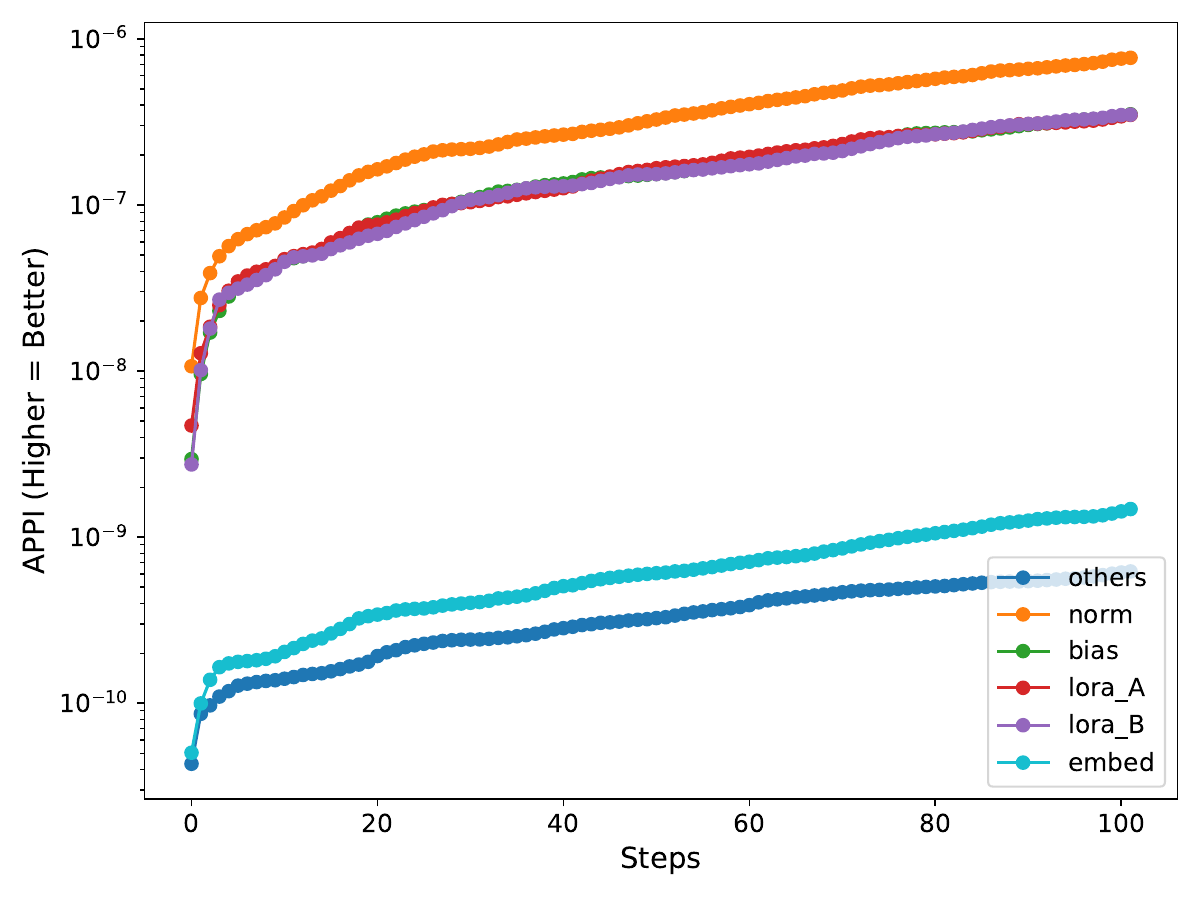}
    \includegraphics[width=0.19\linewidth]{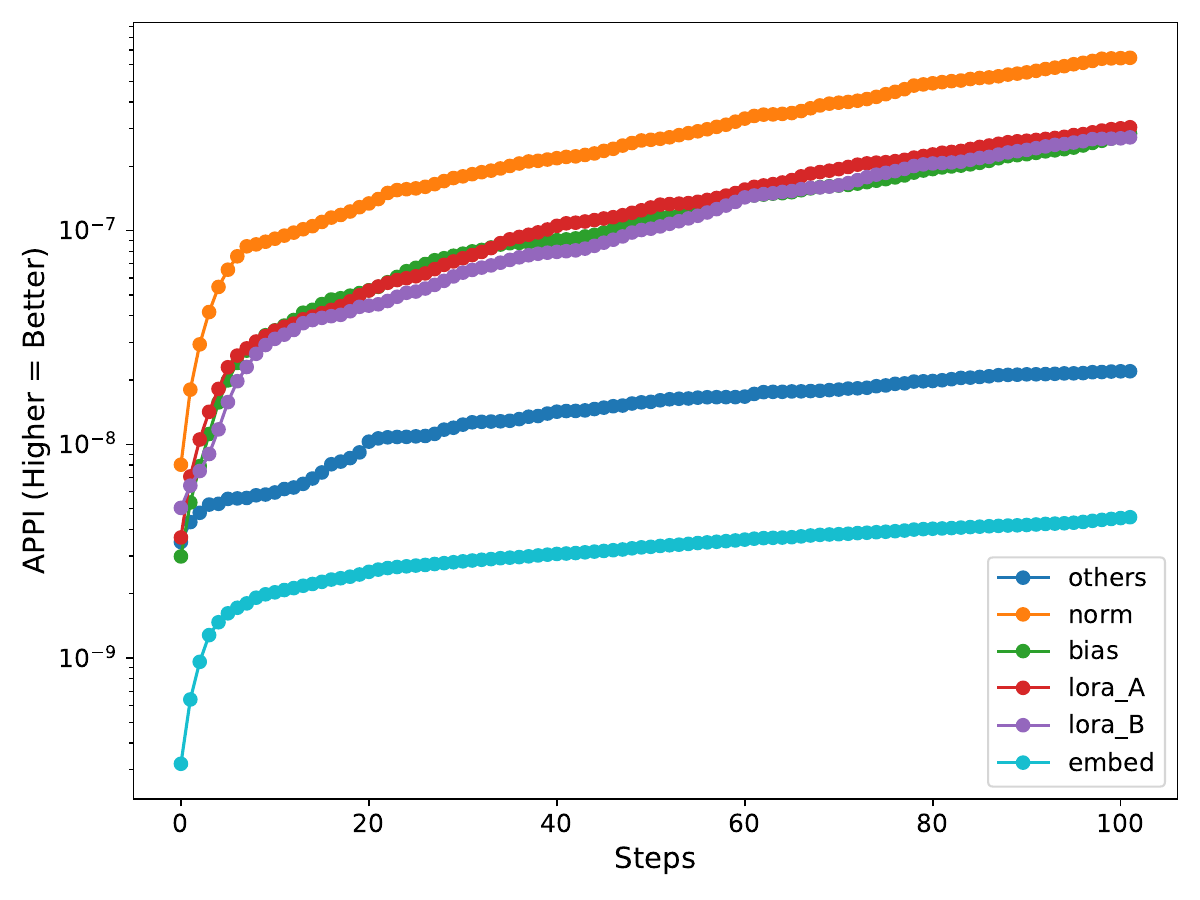}
    \vspace{-0.2cm}
    \caption{Visualization of APPI. Left two on SST2: RoBERTa-base/large. Right three on E2E: GPT2-small/medium/large.
    }
    \label{fig:appi}
\end{figure}

\subsection{Consistency across model sizes}
Furthermore, we observe a roughly consistent pattern across model sizes for the same architecture and task. We vary RoBERTa and T5 sizes in \Cref{fig:heatmap model size}, and GPT2 from small (124M) to large size (0.8B) in \Cref{fig:gpt2}. We additionally vary ViT from tiny (5M) to large size (0.3B) in \Cref{fig:cv} in \Cref{app:exp}. Such observation encourages us to train on small models and directly adopt the optimal PEFT (i.e. the active set of parameter groups) for large models.

\begin{figure}[!htb]
    \centering
    \includegraphics[width=0.24\linewidth]{figs/HMe2e+seed2+gpt2.sm+autolr0.001+optadamw++100.0_nan+heatmap.pdf}
    \includegraphics[width=0.24\linewidth]{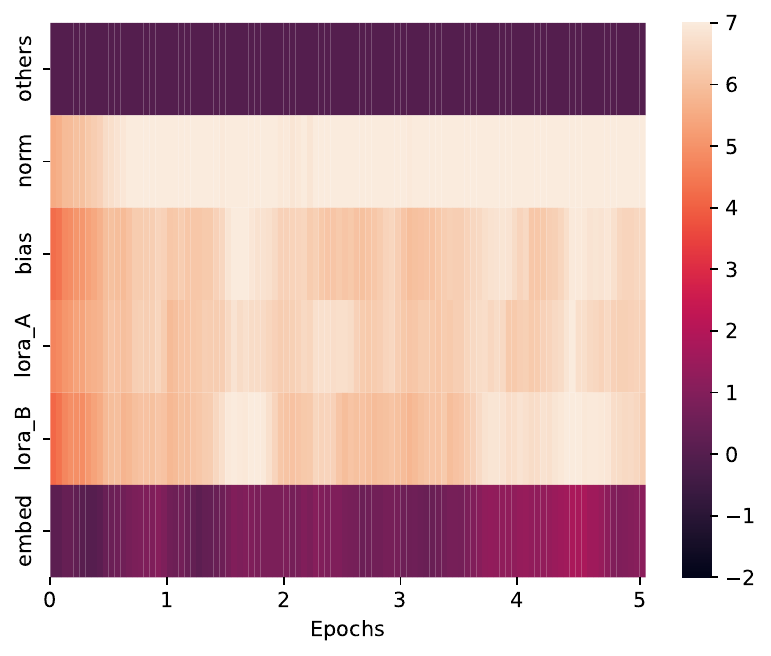}
    \includegraphics[width=0.24\linewidth]{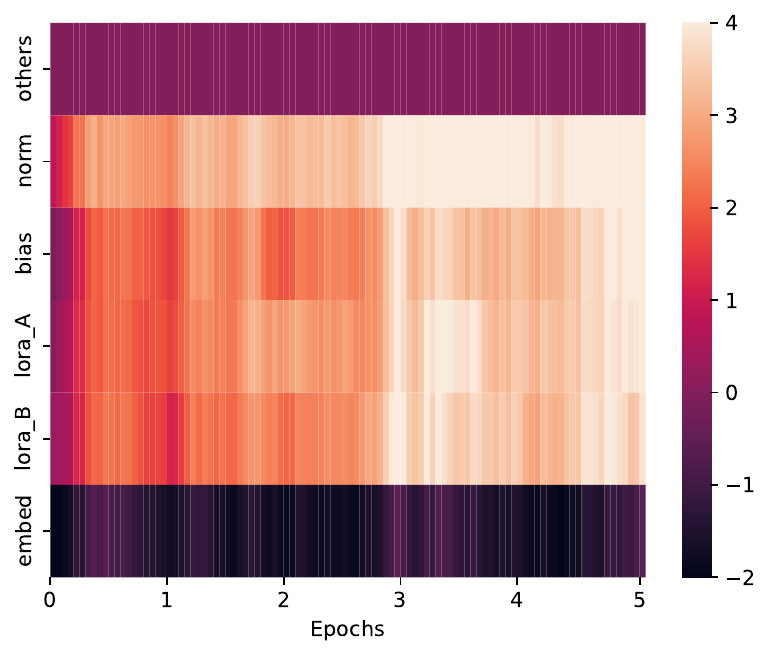}
    \caption{Heatmap of PPI on E2E. Left to right: GPT2 small, medium and large.}
    \label{fig:gpt2}
\end{figure}

\section{Adaptive PEFT framework}
\label{sec:adapeft}
We propose the \textbf{AdaPEFT} framework to adaptively select the trainable parameter groups for PEFT. 

Firstly, we demonstrate that selecting the active set via APPI is approximately Pareto optimal. In \Cref{fig:pareto roberta} (left column), the theoretical Pareto frontier formed by our selection (6 active sets) closely matches that formed by all $2^6=64$ active sets. 

\begin{figure}[!htb]
    \centering
\includegraphics[width=0.27\linewidth]{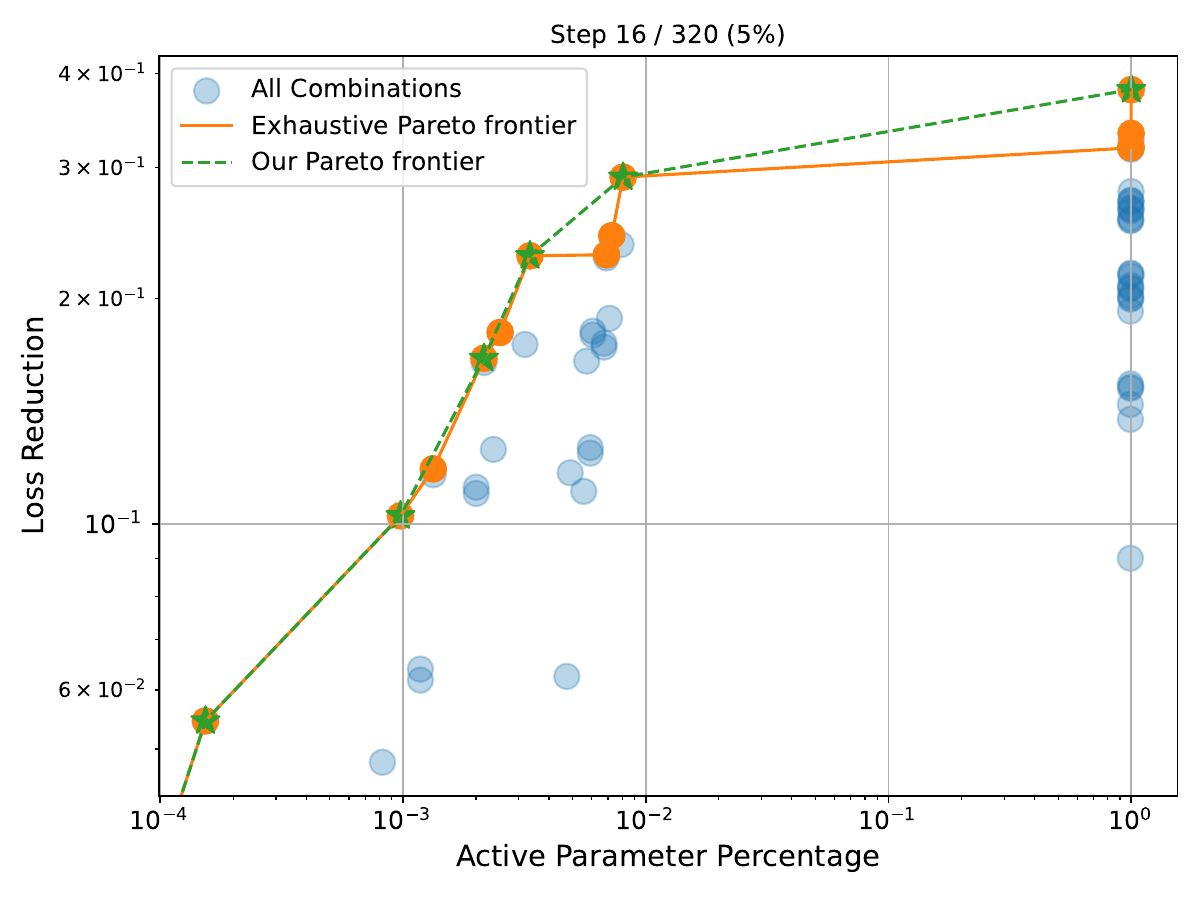}
\includegraphics[width=0.29\linewidth]{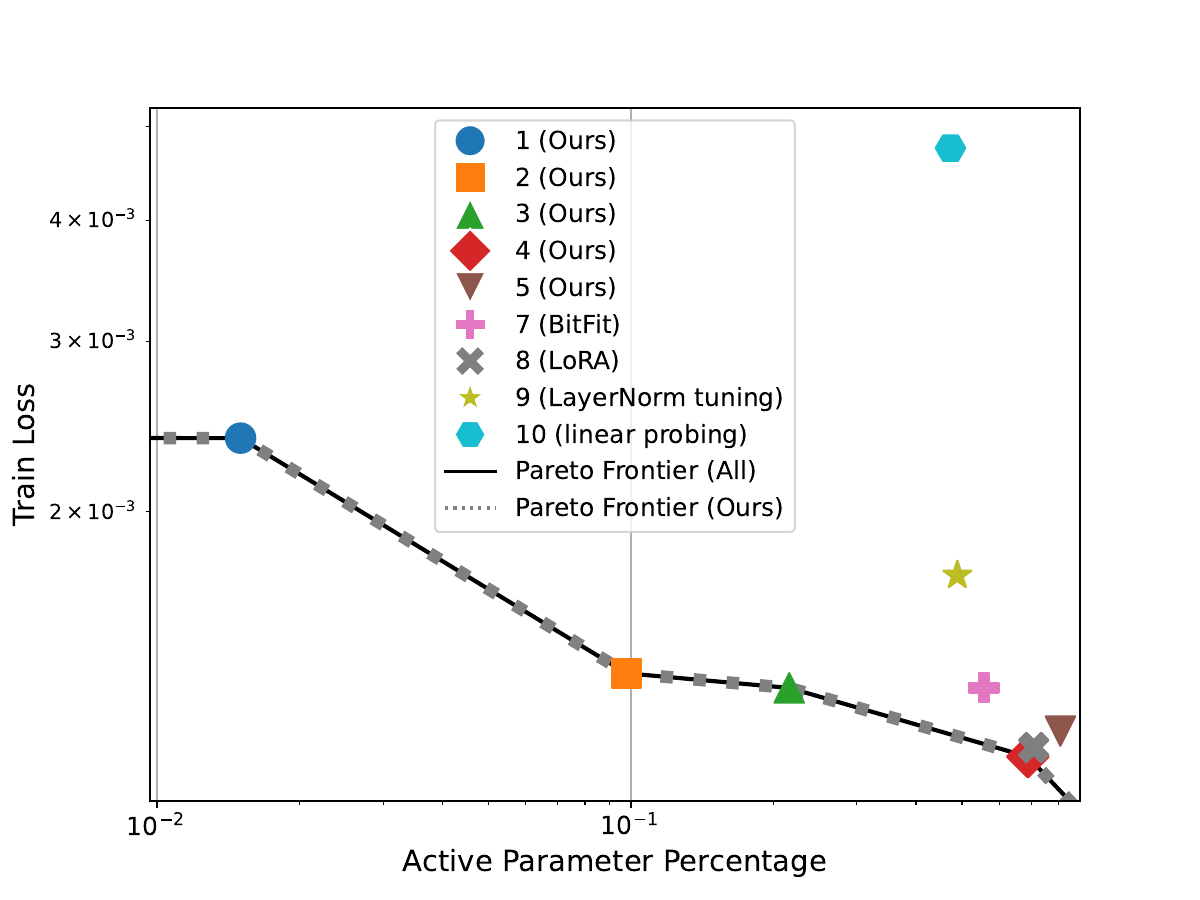}
\includegraphics[width=0.29\linewidth]{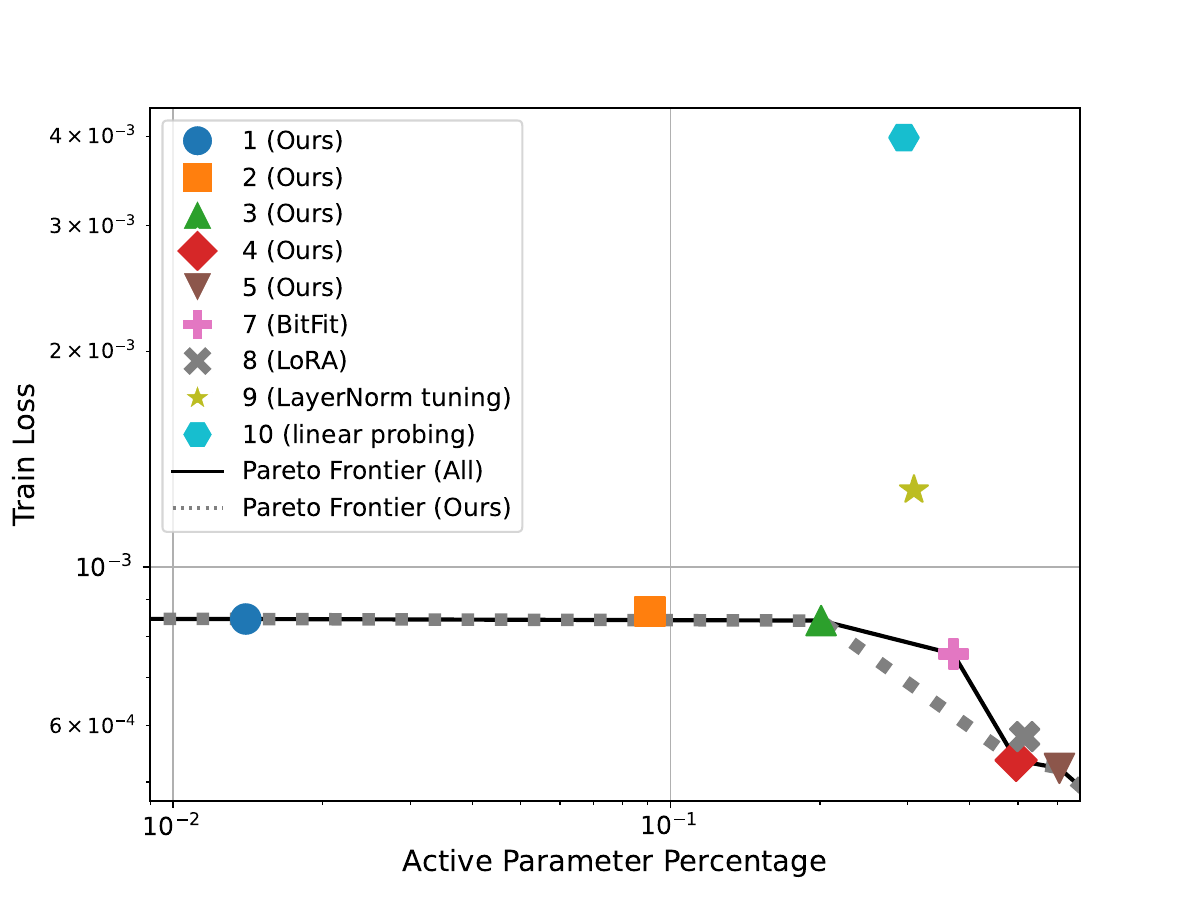}
\\
\includegraphics[width=0.27\linewidth]{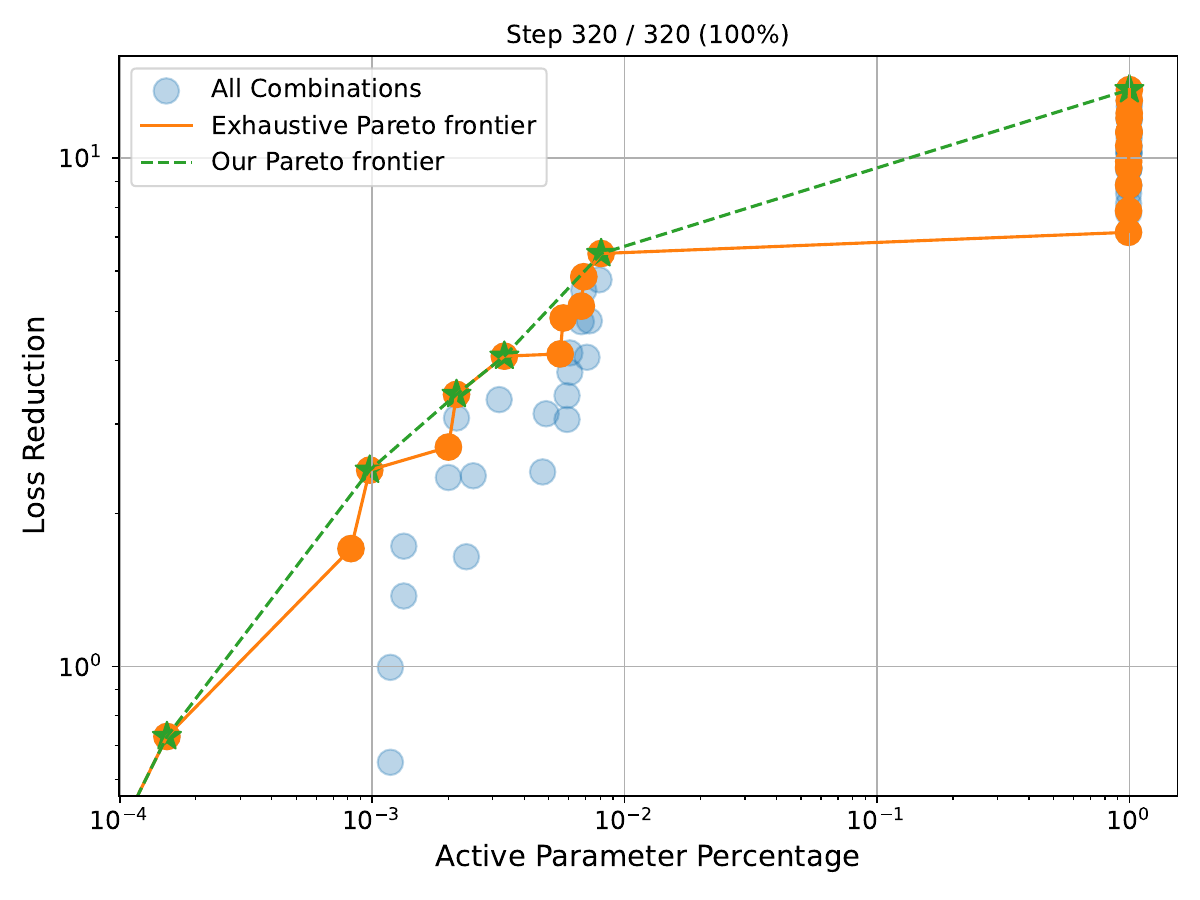}
\includegraphics[width=0.29\linewidth]{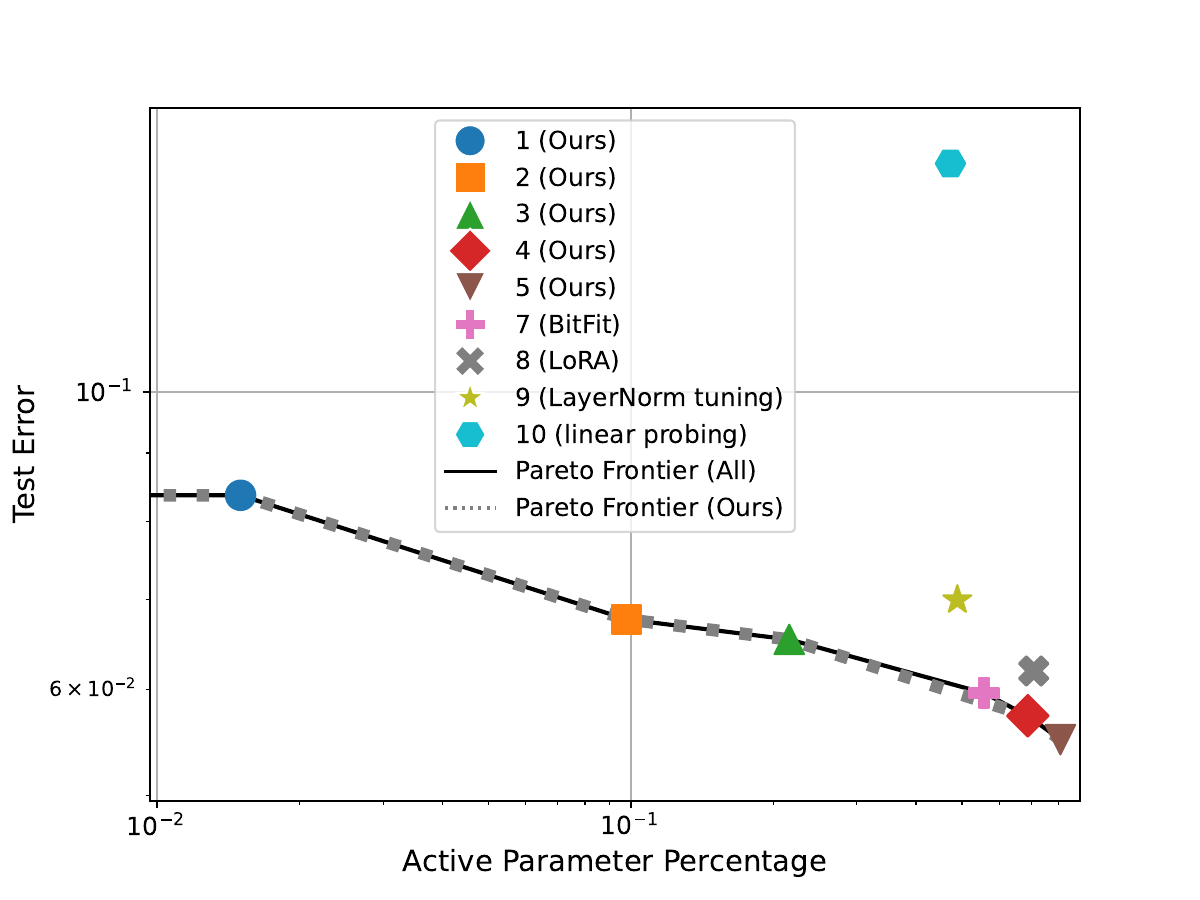}
\includegraphics[width=0.29\linewidth]{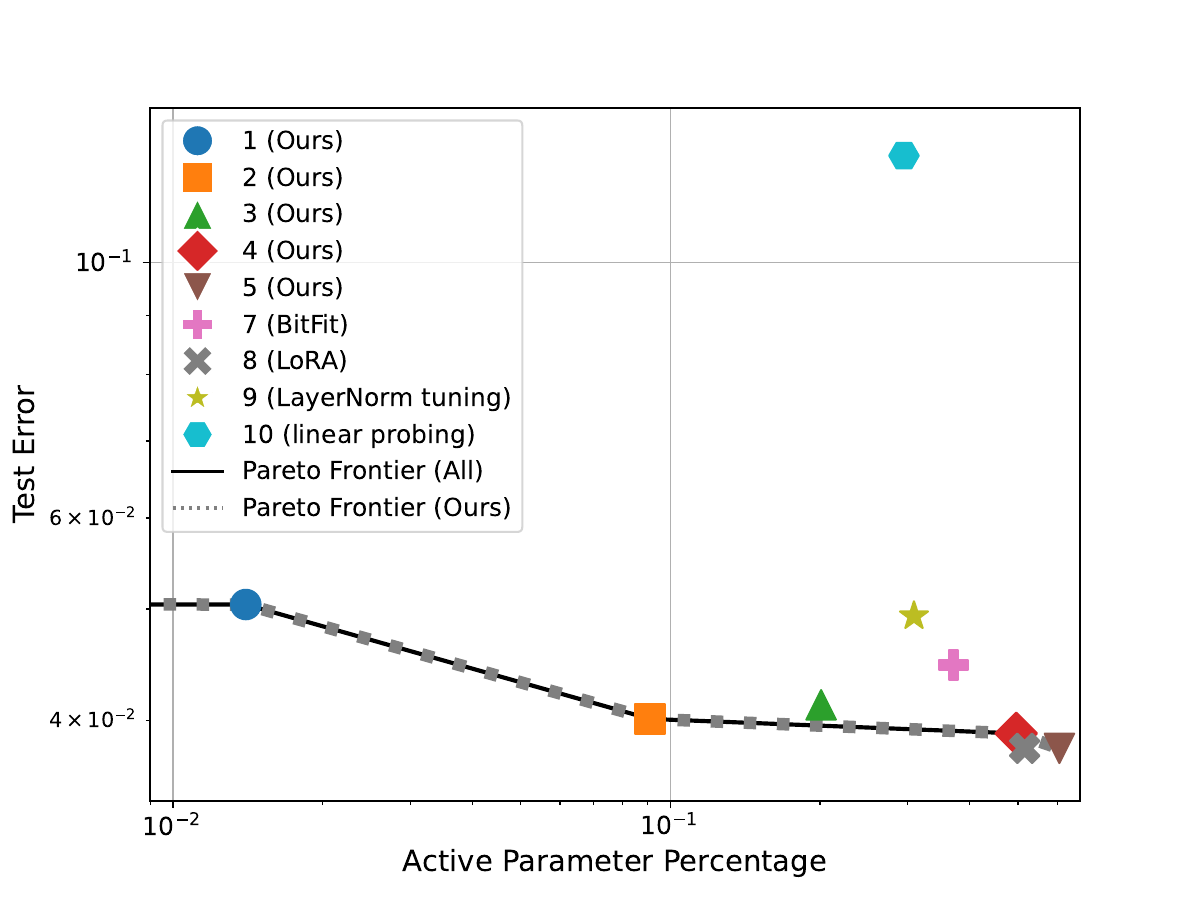}
    \caption{Visualization of Pareto optimality on SST2. Left: theoretical loss reduction of RoBERTa-base via APPI. Middle: actual loss and error of RoBERTa-base. Right: actual loss and error of RoBERTa-large. Each PEFT is indexed in \Cref{tab:RoBERTa PEFT}.}
    \label{fig:pareto roberta}
\end{figure}

Next, we compare our selection (5 active sets excluding FMT) with existing PEFT in terms of actual train loss and test error. In \Cref{fig:pareto roberta} (middle column), our actual Pareto frontier still closely matches the frontier formed by all 10 methods, whereas some existing PEFT methods are far from the frontier, indicating the potential failure of non-adaptive PEFT.

Having validated the (approximate) Pareto optimality of APPI selection, we now give AdaPEFT in \Cref{alg:adapeft}, transferring the active set of parameter groups from (small models, short training) to (large models, long training). Note that AdaPEFT can be implemented with much flexibility, e.g. setting $m=M$ or $10\%\to 100\%$.

\begin{algorithm}[!htb]
\caption{AdaPEFT on model $M$ for $T$ iterations}
\begin{algorithmic}[1]
\State Equip a smaller model $m$ with PEFT components (e.g. LoRA and prefix).
\State Train $m$ with \Cref{alg:autoSGD} under FMT, for 10\% of $T$ iterations.
\State Sort APPI to select influential parameter groups under $|\A|/|\w|\leq\epsilon$ constraint
\State Train $M$ under PEFT with selected $\A$ for $T$ iterations.
\end{algorithmic}
\label{alg:adapeft}
\end{algorithm}

We experiment with \Cref{alg:adapeft} on RoBERTa and GPT2. We select the active sets from small models --  RoBERTa-base and GPT2-small, using 10\% of the training budget, then directly transfer to larger models -- RoBERTa-large, GPT2-medium and GPT2-large. We list the active sets, number of parameters and model utility in \cref{tab:RoBERTa PEFT} and \cref{tab:gpt2 PEFT} (appendix), which are visualized in \Cref{fig:pareto roberta} and \Cref{fig:pareto gpt}.

Our key observation is that (I) the active sets selected by (small model, short horizon) consistently give good Pareto frontier for (large models, long horizon), i.e. AdaPEFT is approximately Pareto optimal and scalable; (II) existing PEFT is either close to our Pareto frontier or far away from the frontier, i.e. AdaPEFT effectively identifies strong PEFT configurations.

\begin{figure}[!htb]
    \centering
\hspace{-0.4cm}
\includegraphics[width=0.29\linewidth]{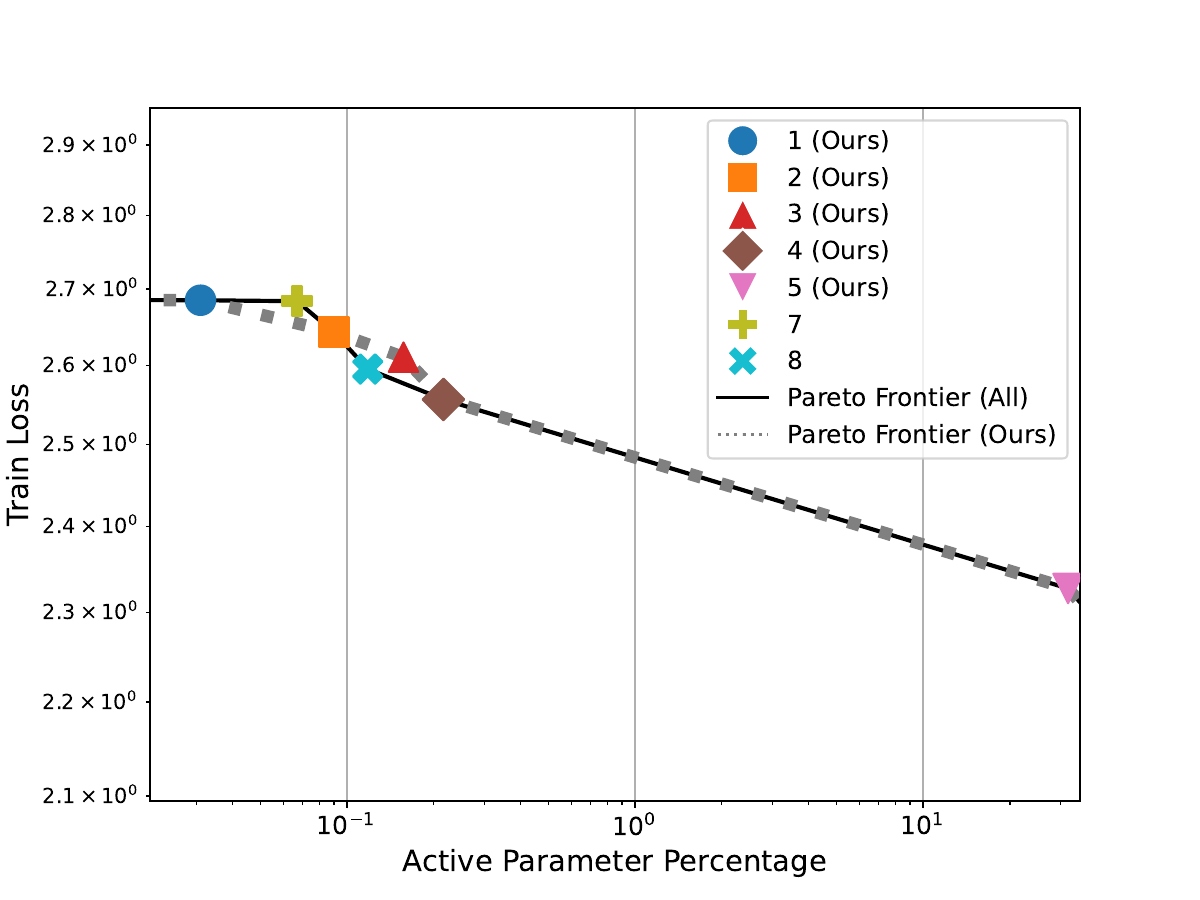}
    \hspace{-0.4cm}
\includegraphics[width=0.29\linewidth]{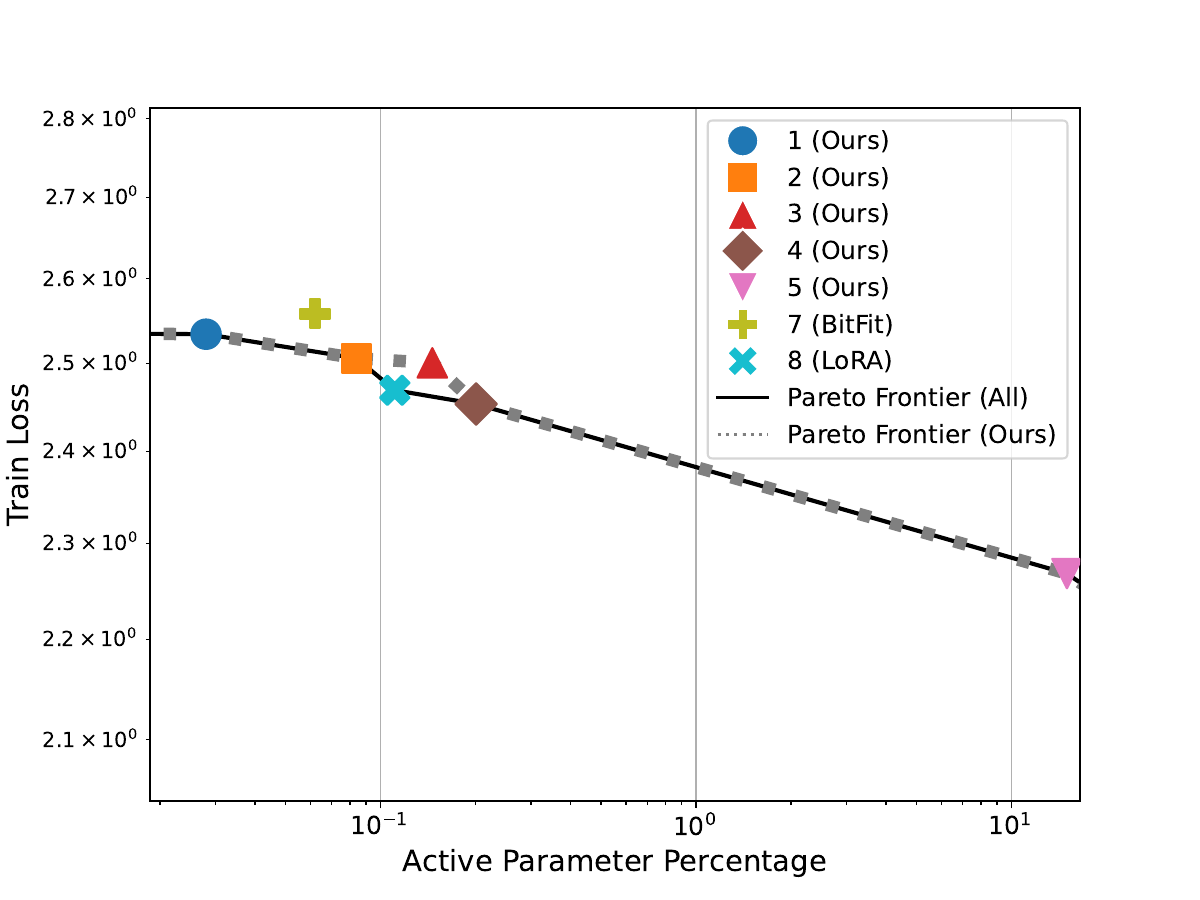}
\hspace{-0.4cm}
\includegraphics[width=0.29\linewidth]{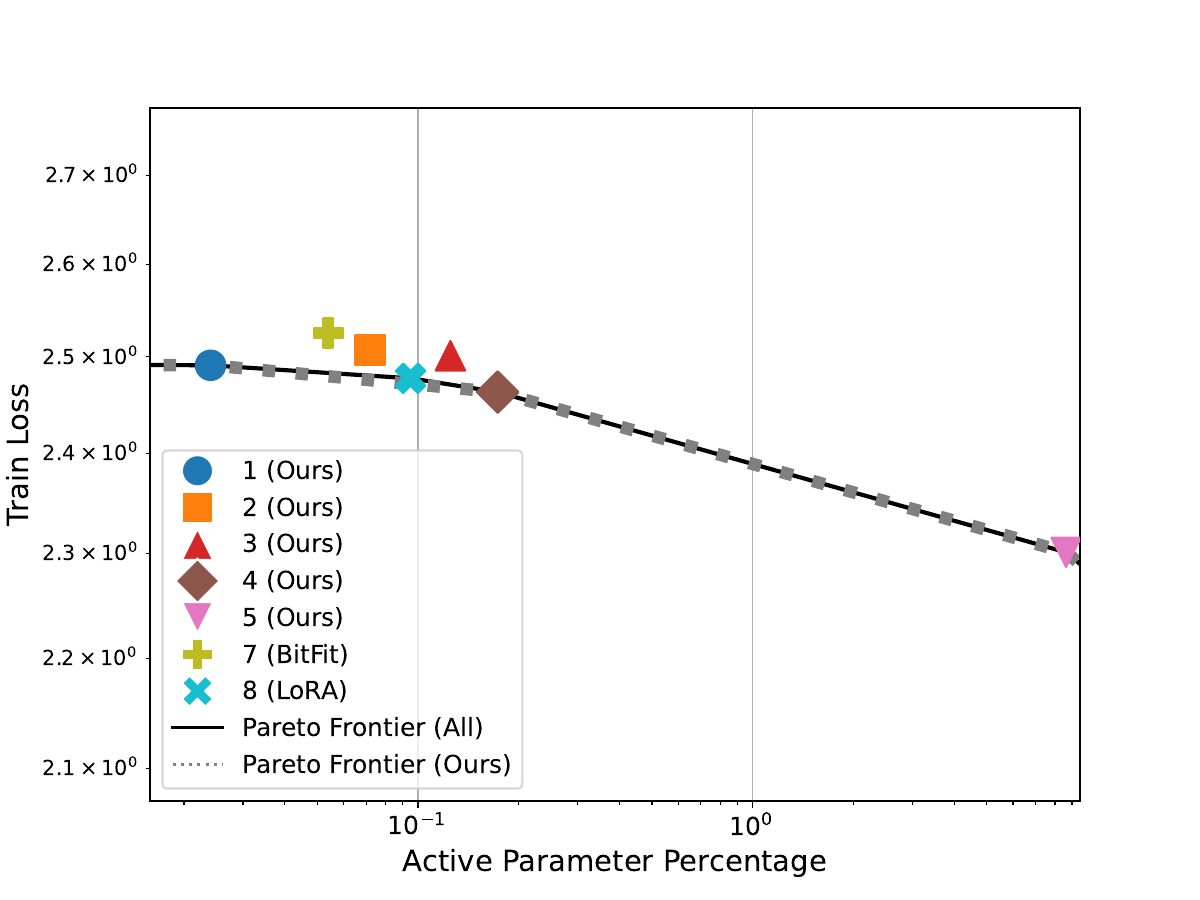}
\\
\hspace{-0.4cm}
\includegraphics[width=0.29\linewidth]{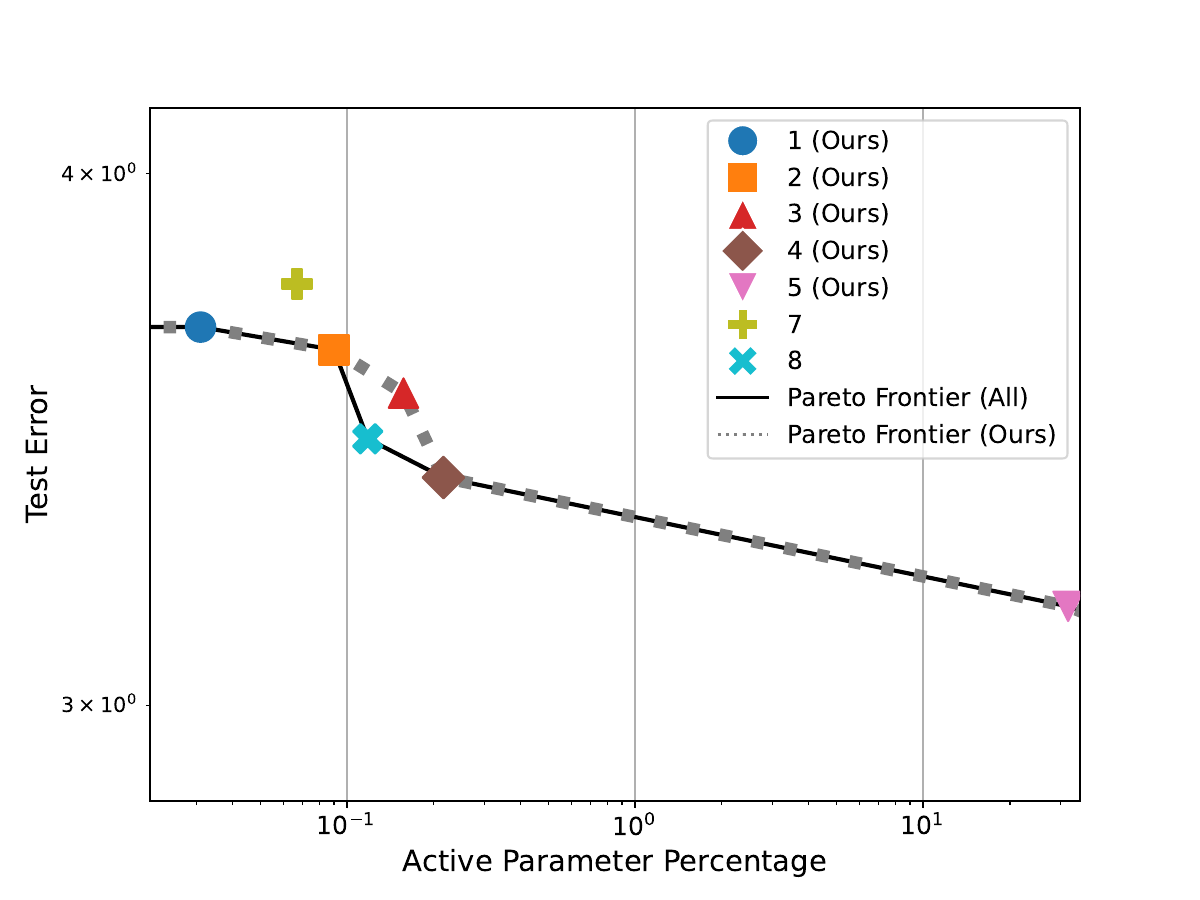}
\hspace{-0.4cm}
\includegraphics[width=0.29\linewidth]{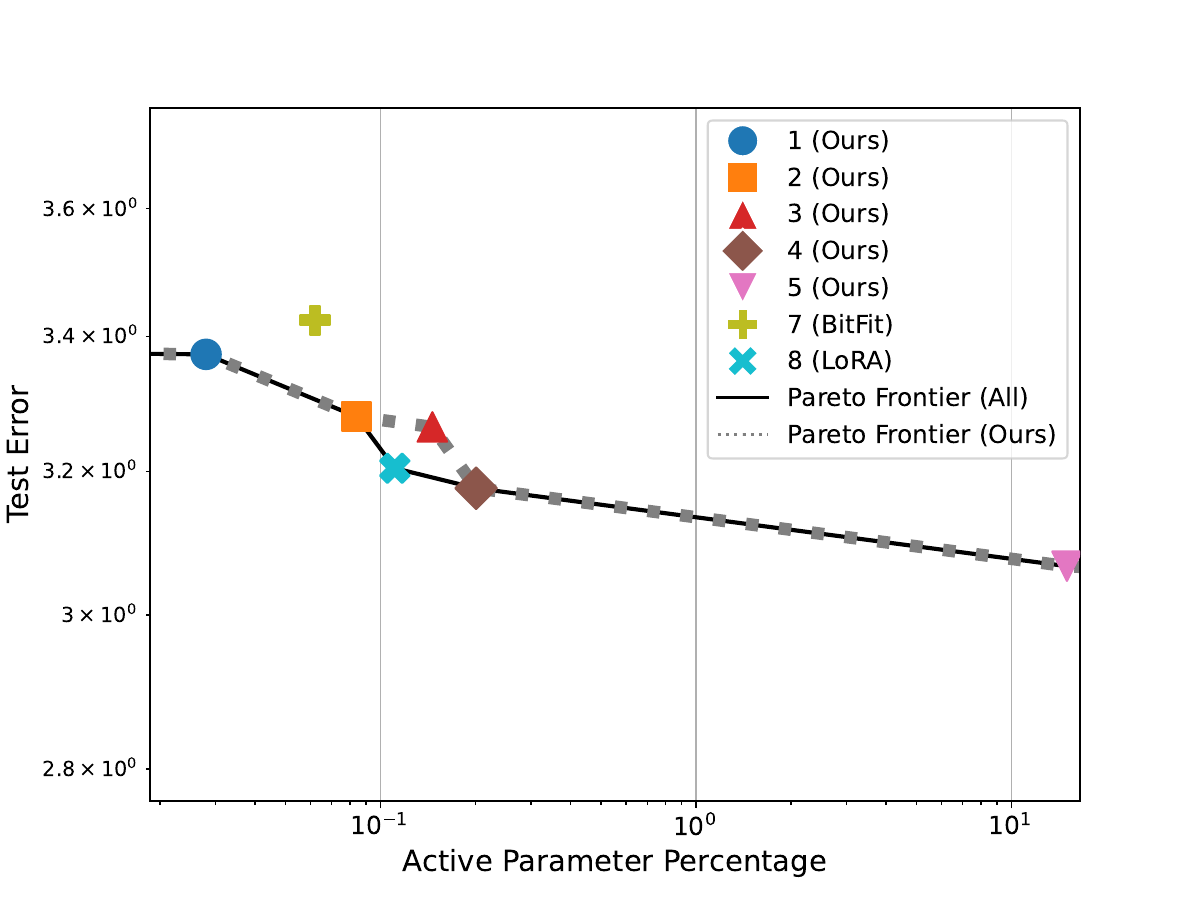}
\hspace{-0.4cm}
\includegraphics[width=0.29\linewidth]{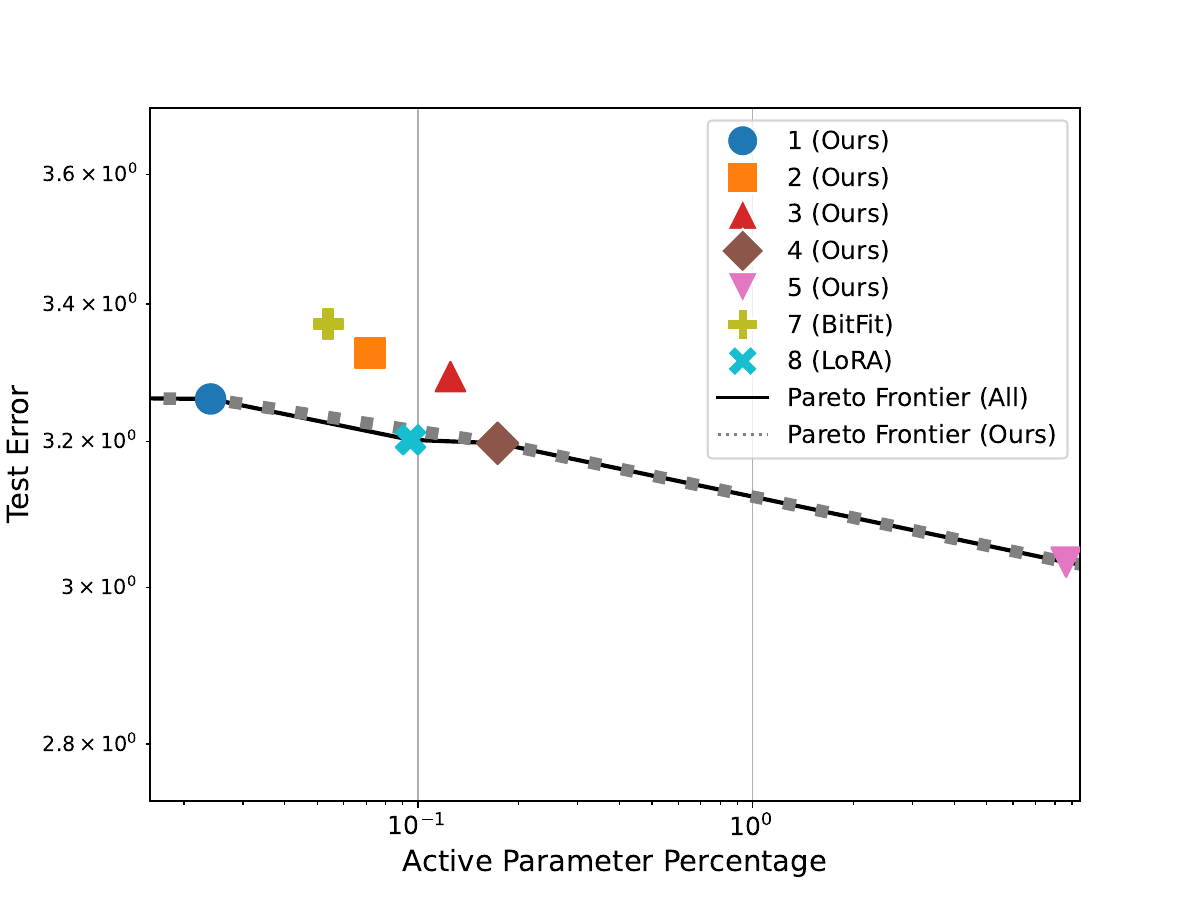}
    \caption{Visualization of Pareto optimality on E2E. Left: actual loss and error of GPT2-small. Right: actual loss and error of GPT2-medium and GPT2-large. Each PEFT is indexed in \Cref{tab:gpt2 PEFT}.}
    \label{fig:pareto gpt}
\end{figure}

\section{Discussion}
We formulate the selection of active set in PEFT as a multi-task optimization problem, and transform it to 0-1 knapsack problem that we solve under Pareto optimality. In particular, our objective in the knapsack problem is Hessian-informed, which demonstrates that different parameters have different influences on the model performance. Finally, we propose AdaPEFT to leverage such influence pattern and select the active set for large model and long training with minimal budget. We note the success of AdaPEFT depends on the grouping of parameters: a sub-optimal grouping strategy may fail to lead to good performance even with AdaPEFT.

\newpage
\clearpage

\bibliography{ref}
\bibliographystyle{plain}


\newpage
\appendix
\onecolumn

\section{Experiment details}\label{app:exp}
Our experiments are run on A100 GPU, though our approach is independent to the choice of device. Our experiments are three steps: (I) Visualizing the influence under FMT, trained with \Cref{alg:autoSGD} which is updated every 16 iterations (lazy updating); (II) Selecting influential parameter groups to determine PEFT configurations; (III) Training PEFT with GeN AdamW \cite{bu2024automatic} (lazy updating frequence 8). For hyper-parameters not mentioned here, we follow \cite{hu2022lora}. 

\subsection{Visualization methodology}
\label{app:visual}
Deep learning stochastic optimization is highly non-convex and may be unstable. In addition, the curve fitting approach may occasionally have numerical errors. Therefore, we adopt some outlier removal and smoothing tricks to give reproducible and clear patterns.

For each group $k$ (i.e., a row), our input is a time series of $\text{PPI}_k(t)=\frac{(\G_{(k),t}^\top\g_{(k),t})^2}{\g_{(k),t}^\top\H_{(k),t}\g_{(k),t}\cdot |\w_{(k)}|}$. 

For heatmaps and APPI plots, we remove outliers by Interquartile Range (IQR) method \footnote{The outliers are removed by excluding values outside the range $[Q1-3*iqr, Q3+3*iqr]$, where the interquartile range $iqr$ is defined as the difference between the 25-th percentile $Q1$ and the 75-th percentile $Q3$ of the data, representing the spread of the central 50\%.}, smoothen by exponential moving average. For heatmaps, we additionally divide each row by the first row (the \textit{others}, which is majority of parameters). Hence, the first row always stands for one unit of influence. As shown below, our methodology is robust to random seeds in terms of ranking.
\begin{figure}[!htb]
\centering
\includegraphics[width=0.24\linewidth]{figs/HMsst2+seed22+roberta-base+autolr2e-05_nan+heatmap.pdf}
\includegraphics[width=0.24\linewidth]{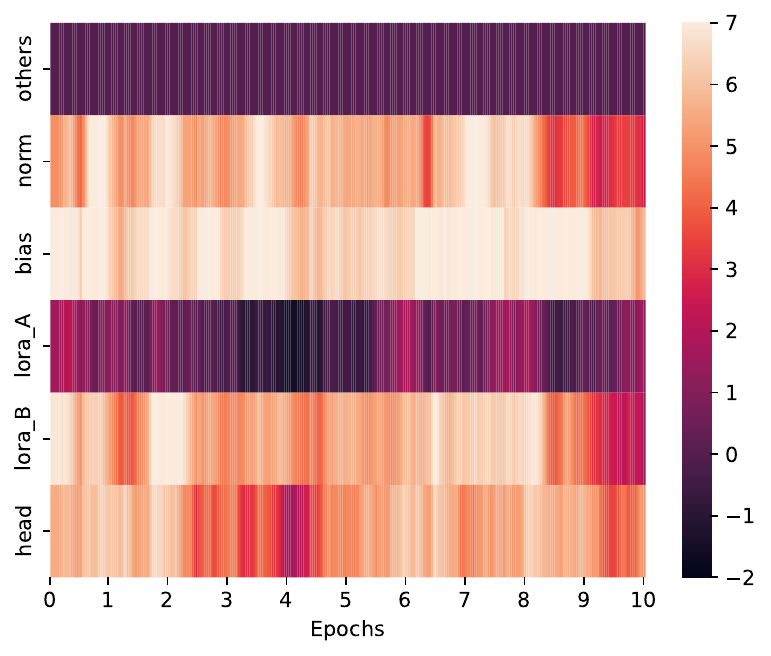}
\includegraphics[width=0.24\linewidth]{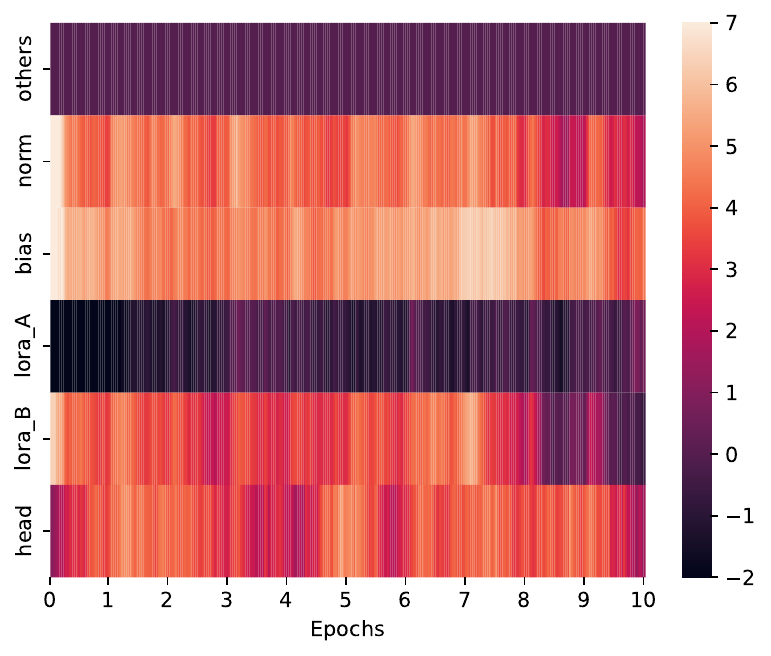}
\includegraphics[width=0.24\linewidth]{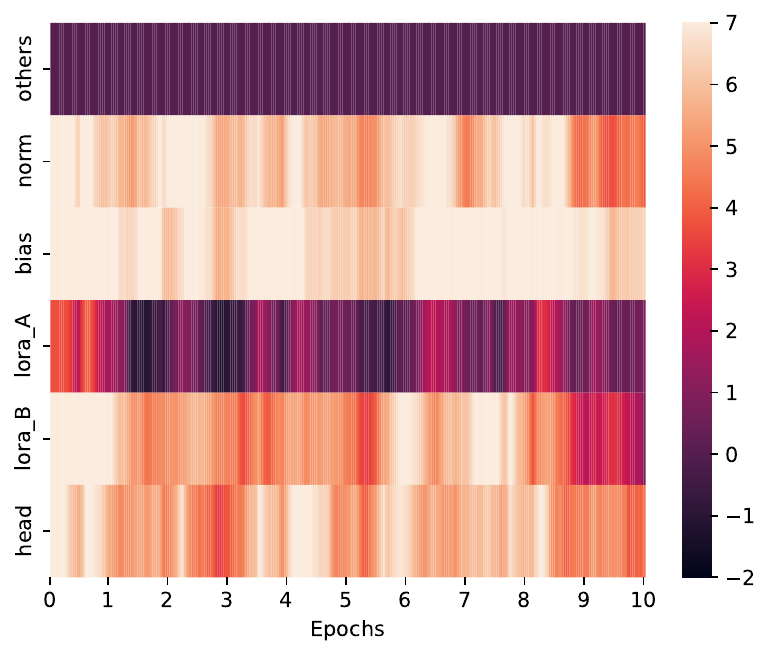}
\caption{Heatmap of PPI with RoBERTa-base on SST2 dataset from different random seeds.}
\label{fig:heatmap seeds}
\end{figure}

\subsection{Natural language understanding}\label{app:lora}

For NLU tasks, we use batch size 128 and an initial learning rate 2e-5. For CoLA training with RoBERTa, we use a total of 3 epochs. For the rest, we use total epochs of 10. The evaluation metric is test accuracy.



   

\subsection{GPT2}\label{app:nlg}
For GPT2, we experiment on the E2E dataset. For FMT with GeN AdamW, we use initial learning rate 1e-4; for PEFT, it is 1e-3. The sequence length is 128, the total batch size is 256. The total number of epochs for GPT2 (small, medium and large) is 5.



\subsection{ViT classification}
We use ImageNet pre-trained ViT \cite{dosovitskiy2020image}, which can be loaded
from \texttt{timm} library. We resize all images to 224x224 and normalize the pixel values to [-1,1]. 
We use initial learning rate 1e-4. We apply \Cref{alg:autoSGD} every 16 iterations.

\begin{figure}[H]
    \centering
    \includegraphics[width=0.24\linewidth]{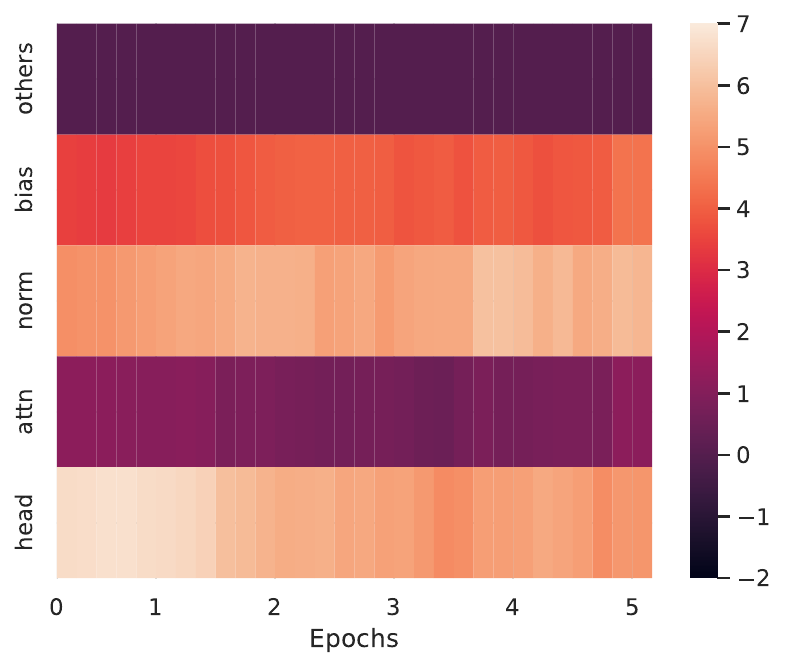}
    \includegraphics[width=0.24\linewidth]{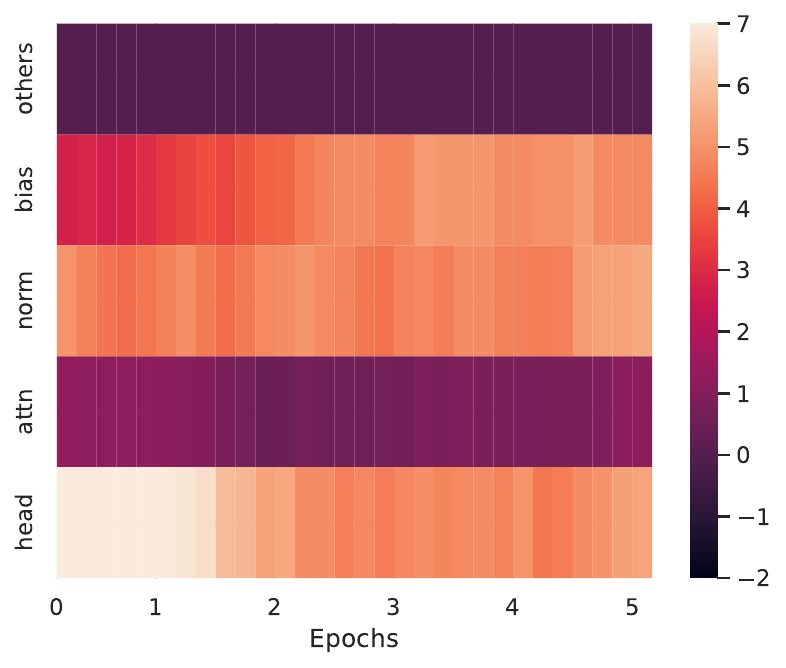}
    \includegraphics[width=0.24\linewidth]{figs/HMCIFAR100+seed1+e5+vit_base_patch16_224+autolr0.0001_nan+heatmap.pdf}
    \includegraphics[width=0.24\linewidth]{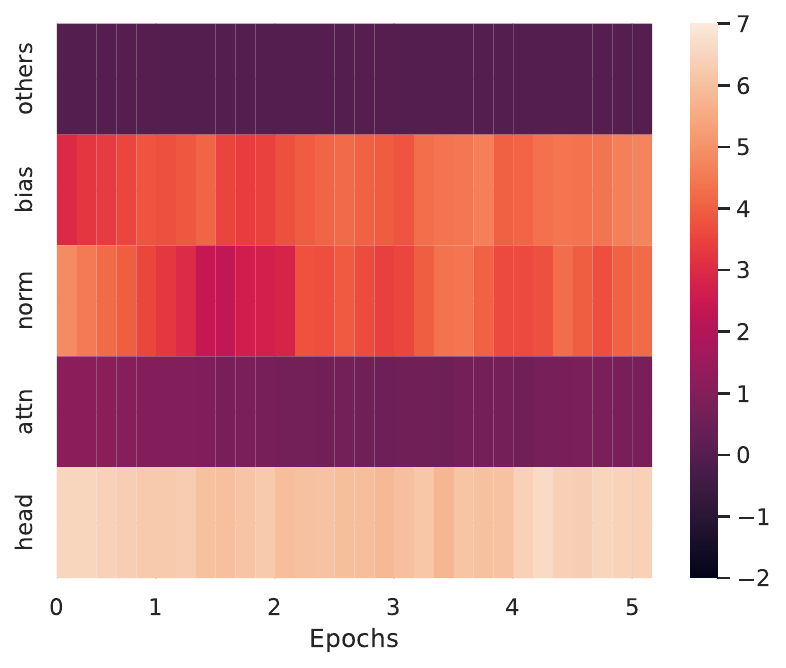}
    \caption{Heatmap of PPI on CIFAR100. Left to right: ViT tiny, small, base and large. }
    \label{fig:cv}
\end{figure}





\section{Extended experiments}
\label{app:new exp}
We extend \Cref{fig:pareto roberta} below. Empirically speaking, our AdaPEFT Pareto frontier generated from 6 active sets matches closely the frontier generated from $2^6=64$ active sets (all possible combinations) throughout the training.

\begin{figure}[!htb]
    \centering
    \includegraphics[width=0.23\linewidth]{figs/HMsst2+seed22+roberta-base+autolr2e-05_nan+theory_pareto_5}
    \includegraphics[width=0.23\linewidth]{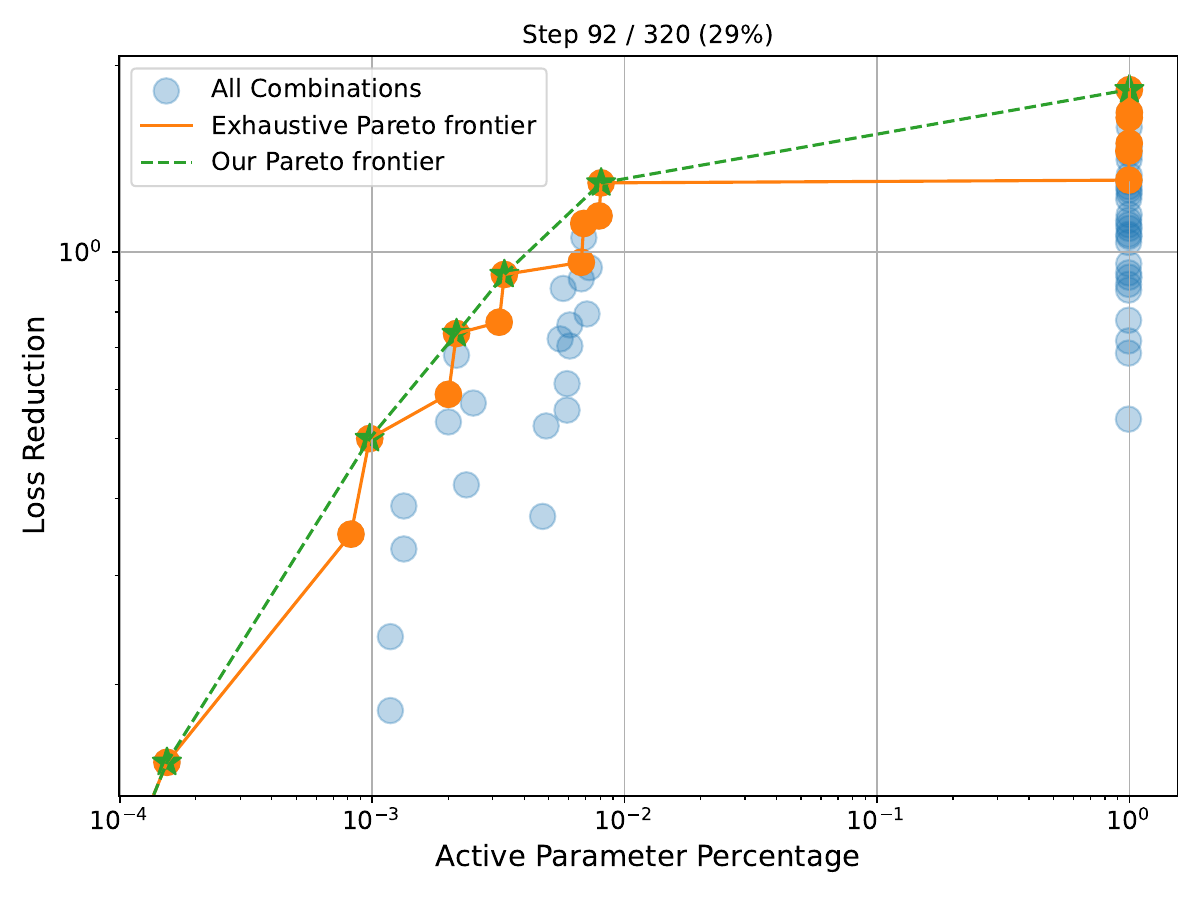}    \includegraphics[width=0.23\linewidth]{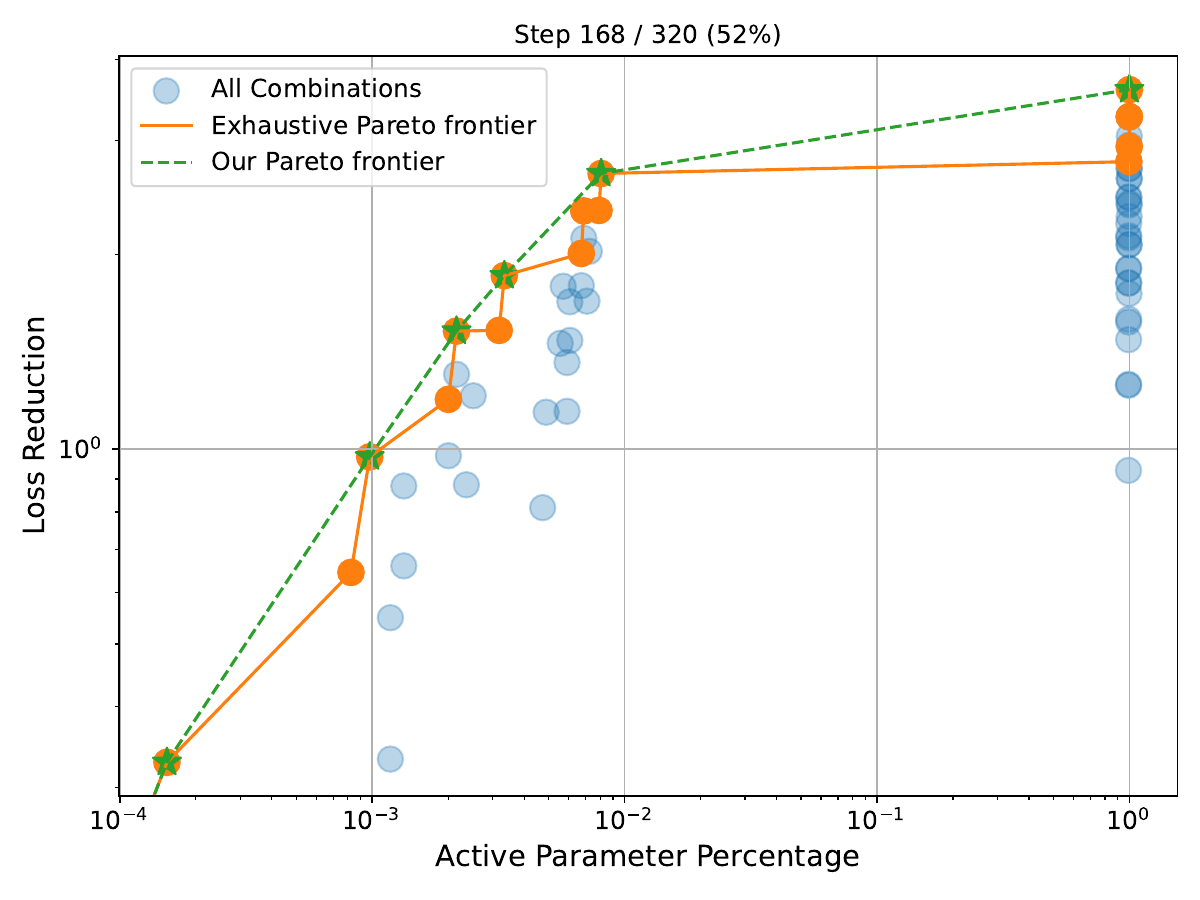}
    \\
    \includegraphics[width=0.23\linewidth]{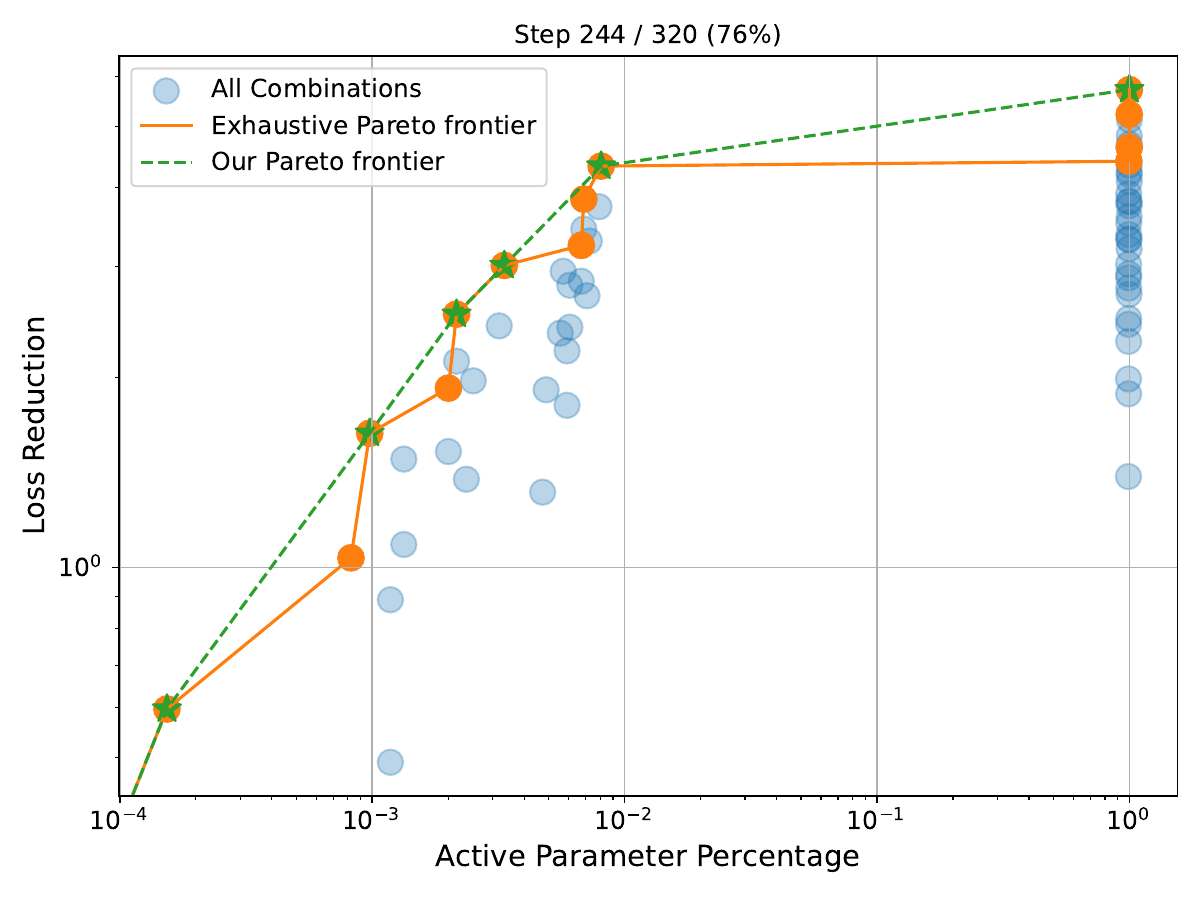}    \includegraphics[width=0.23\linewidth]{figs/HMsst2+seed22+roberta-base+autolr2e-05_nan+theory_pareto_100}
    \caption{Loss reduction in APPI (see \eqref{eq:PPI}) at different iterations in training. SST2 with RoBERTa-base. 5 epochs, 5270 iterations, logged every 16 iterations by lazy updating ($5270/16\approx 320$). Dark blue color is due to overlapping.}
\end{figure}

Furthermore, we reproduce the success of AdaPEFT Pareto frontier on two different architectures, two datasets, and various model sizes.

\begin{figure}[!htb]
    \centering
\begin{subfigure}{0.3\linewidth}
\includegraphics[width=\linewidth]{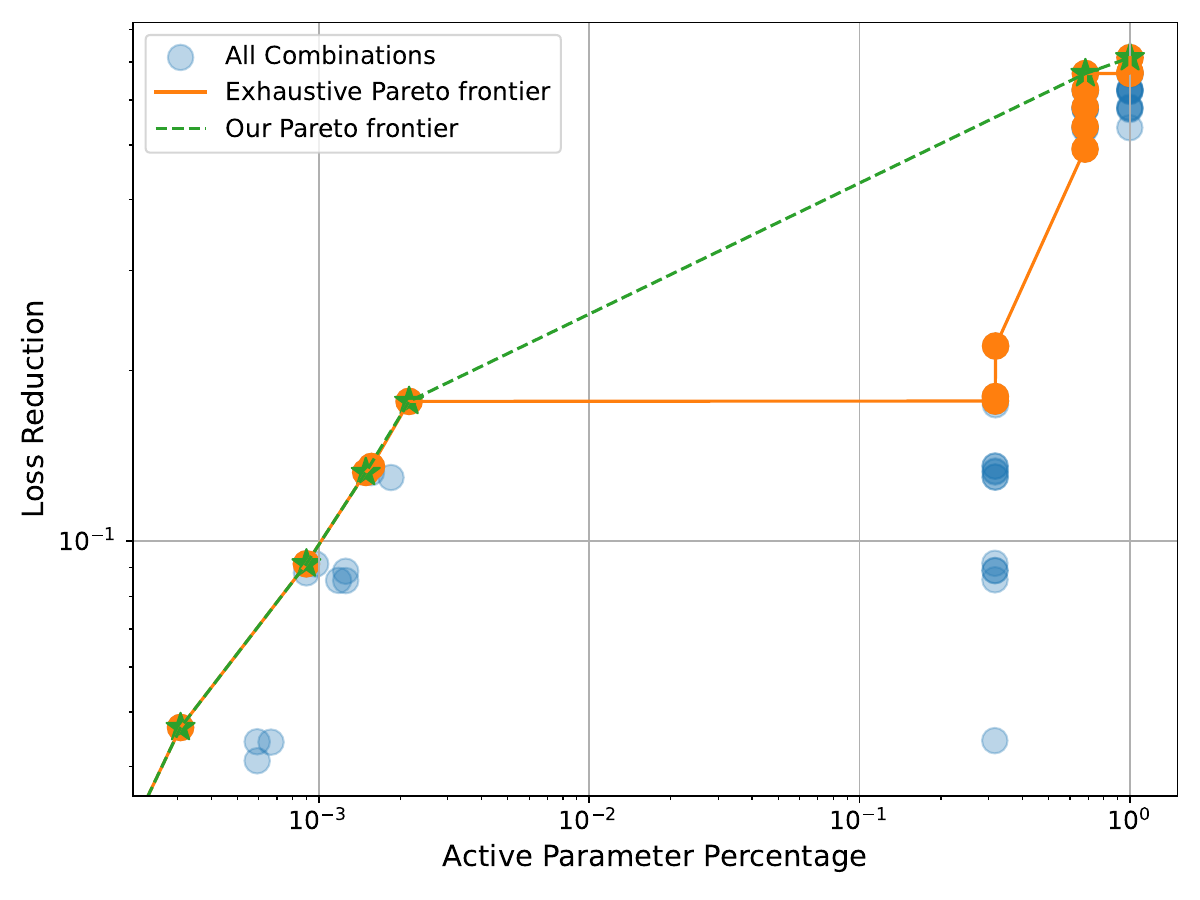}
\caption{GPT2-small on E2E}
\end{subfigure}
\begin{subfigure}{0.3\linewidth}
\includegraphics[width=\linewidth]{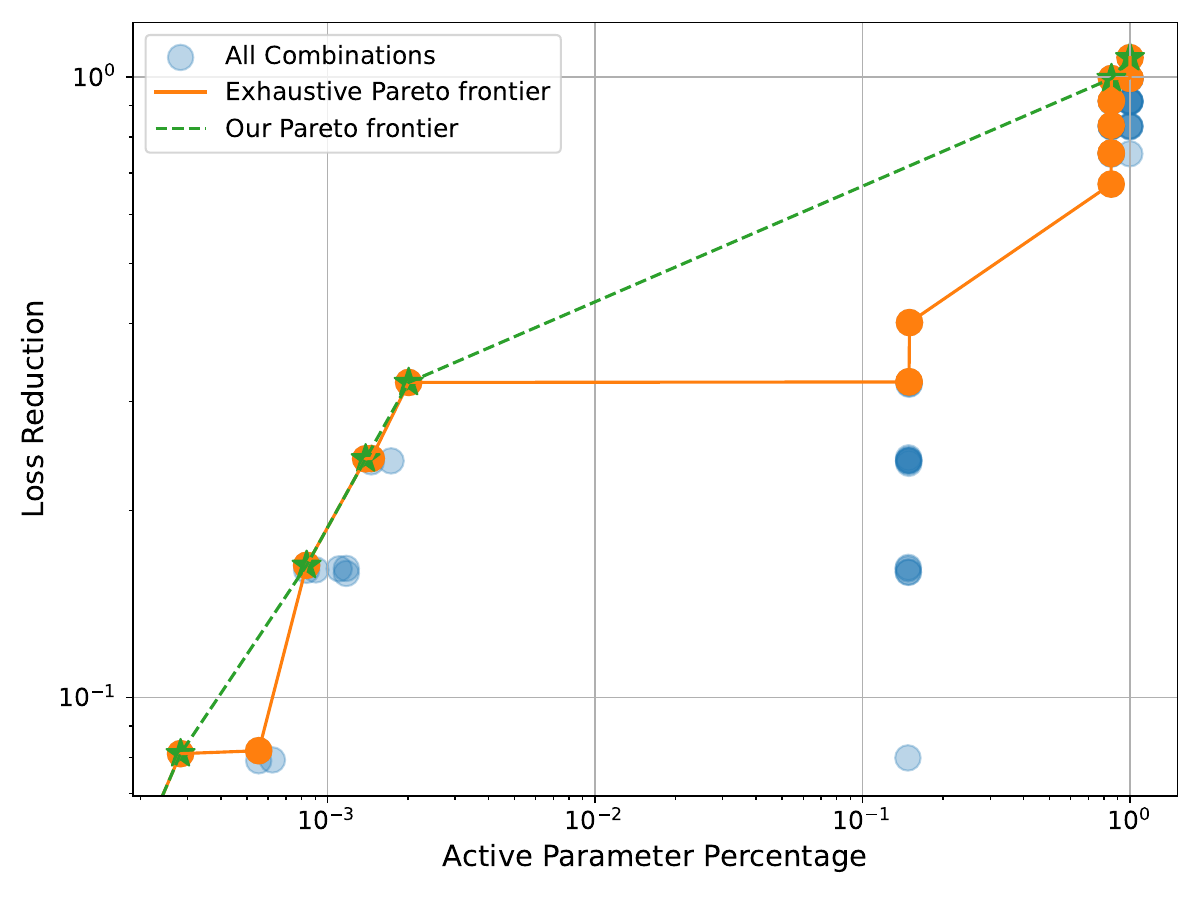}
\caption{GPT2-medium on E2E}
\end{subfigure}
\begin{subfigure}{0.3\linewidth}
\includegraphics[width=\linewidth]{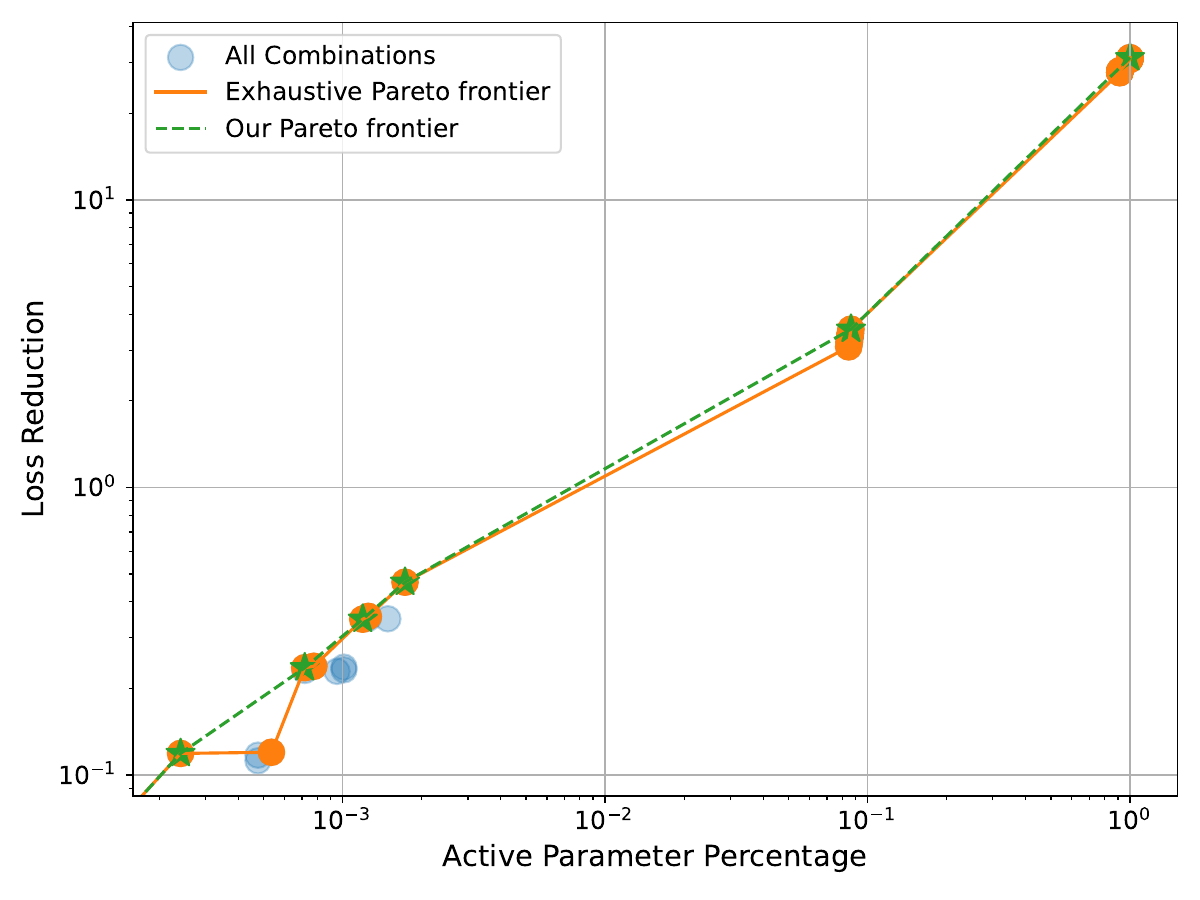}
\caption{GPT2-large on E2E}
\end{subfigure}
\begin{subfigure}{0.3\linewidth}
\includegraphics[width=\linewidth]{figs/HMsst2+seed22+roberta-base+autolr2e-05_nan+theory_pareto_100}
\caption{RoBERTa-base on SST2}
\end{subfigure}
\begin{subfigure}{0.3\linewidth}
\includegraphics[width=\linewidth]{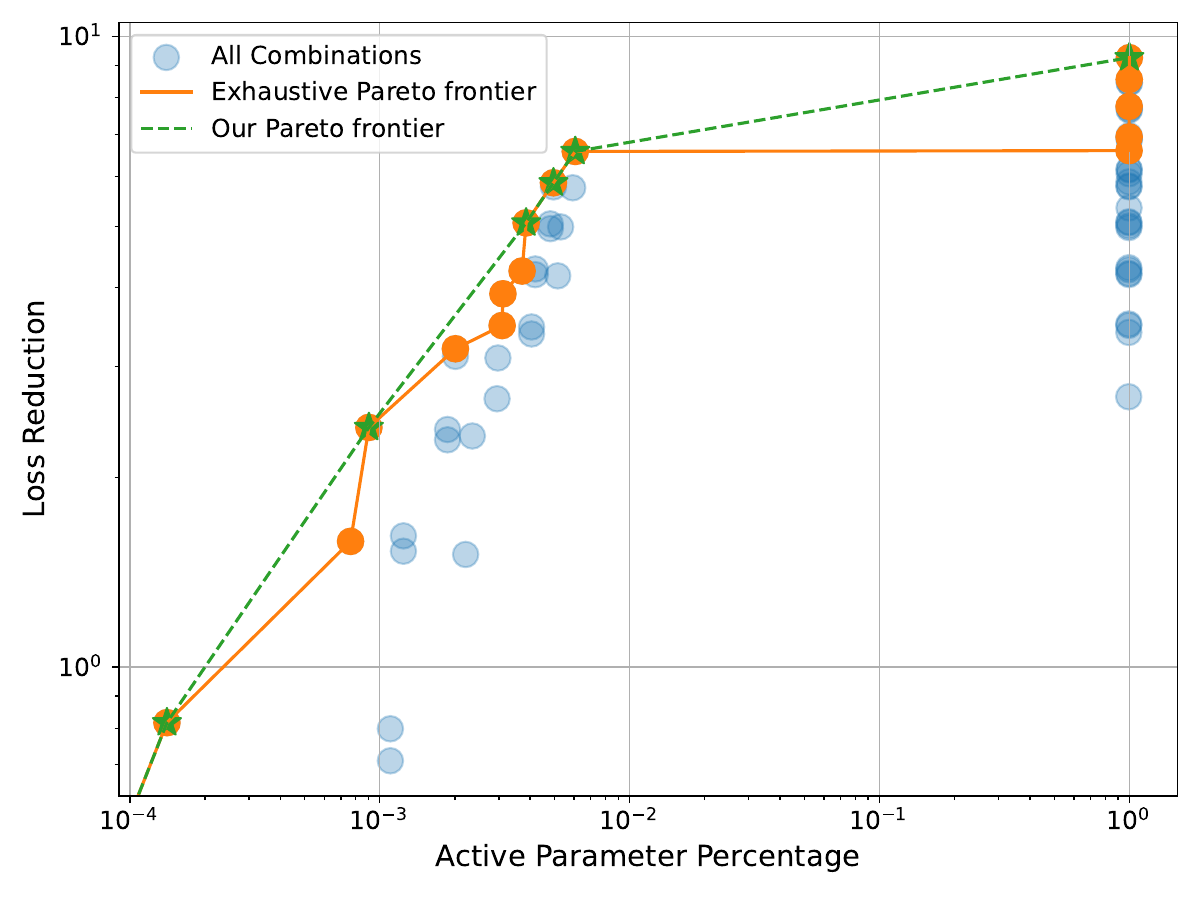}
\caption{RoBERTa-large on SST2}
\end{subfigure}
\caption{Loss reduction in APPI (see \eqref{eq:PPI}) at the last iteration. Dark blue color is due to overlapping.}
\end{figure}

\clearpage
\section{Tables}
\label{app:tables}
\begin{table}[!htb]
\centering
\caption{Performance of RoBERTa models on SST2. (Y)es indicates a parameter group is trainable. (N)o indicates a group is frozen.}    \resizebox{\linewidth}{!}{
    \begin{tabular}{c|c||c|c|c|c|c|c||c|c}
        model &codename  &others&norm&bias&lora\_A&lora\_B&head&accuracy&\% param\\\hline
         RoBERTa-base&1&\N&\Y&\N&\N&\N&\N&91.63&0.015 \\
         (ours)&2&\N&\Y&\Y&\N&\N&\N&93.23&0.098\\
         &3&\N&\Y&\Y&\N&\Y&\N&93.46&0.216\\
         &4&\N&\Y&\Y&\N&\Y&\Y&94.27&0.689\\
         &5&\N&\Y&\Y&\Y&\Y&\Y&94.50&0.807\\
         &FMT&\Y&\Y&\Y&\Y&\Y&\Y&93.35&100
         \\\hline
         RoBERTa-base& 7(BitFit)&\N&\N&\Y&\N&\N&\Y&94.04&0.556\\
         (heuristic)&8(LoRA)&\N&\N&\N&\Y&\Y&\Y&93.81&0.709\\
         &9(LayerNorm)&\N&\Y&\N&\N&\N&\Y&93.00&0.489\\
        &10(linear probing)&\N&\N&\N&\N&\N&\Y&85.21&0.473\\
         \hline\hline
         RoBERTa-large&1&\N&\Y&\N&\N&\N&\N&94.95&0.014 \\
         (ours)&2&\N&\Y&\Y&\N&\N&\N&95.99&0.091\\
         &3&\N&\Y&\Y&\N&\Y&\N&95.87&0.201\\
         &4&\N&\Y&\Y&\N&\Y&\Y&96.10&0.496\\
         &5&\N&\Y&\Y&\Y&\Y&\Y&96.22&0.606\\
         &FMT&\Y&\Y&\Y&\Y&\Y&\Y&95.99&100
         \\\hline
         RoBERTa-large& 7(BitFit)&\N&\N&\Y&\N&\N&\Y&95.53&0.371\\
         (heuristic)&8(LoRA)&\N&\N&\N&\Y&\Y&\Y&96.22&0.516\\
         &9(LayerNorm)&\N&\Y&\N&\N&\N&\Y&95.07&0.309\\
        &10(linear probing)&\N&\N&\N&\N&\N&\Y&87.61&0.295\\
         \hline\hline
    \end{tabular}
}
    \label{tab:RoBERTa PEFT}
\end{table}

\begin{table}[!htb]
    \centering
\caption{Performance of GPT models on E2E. (Y)es indicates a parameter group is trainable. (N)o indicates a group is frozen. We transfer the PEFT identified at $\psi=10$ to larger models.}    \resizebox{\linewidth}{!}{
    \begin{tabular}{c|c||c|c|c|c|c|c||c|c}
        model &codename  &others&norm&bias&lora\_A&lora\_B&embed&perplexity&\% param\\\hline
         GPT2-small&1&\N&\Y&\N&\N&\N&\N&3.68&0.031 \\
         (ours)&2&\N&\Y&\N&\N&\Y&\N&3.63&0.09\\
         &3&\N&\Y&\Y&\N&\Y&\N&3.55&0.157\\
         &4&\N&\Y&\Y&\Y&\Y&\N&3.39&0.216\\
         &5&\N&\Y&\Y&\Y&\Y&\Y&3.16&31.827\\
         &FMT&\Y&\Y&\Y&\Y&\Y&\Y&3.07&100
         \\\hline
         GPT2-small& 7(BitFit)&\N&\N&\Y&\N&\N&\N&3.76&0.067\\
         (heuristic)&8(LoRA)&\N&\N&\N&\Y&\Y&\N&3.46&0.118\\
        \hline\hline
         GPT2-medium&1&\N&\Y&\N&\N&\N&\N&3.37&0.028 \\
         (ours)&2&\N&\Y&\N&\N&\Y&\N&3.27&0.084\\
         &3&\N&\Y&\Y&\N&\Y&\N&3.26&0.146\\
         &4&\N&\Y&\Y&\Y&\Y&\N&3.17&0.201\\
         &5&\N&\Y&\Y&\Y&\Y&\Y&3.06&14.984\\
         &FMT&\Y&\Y&\Y&\Y&\Y&\Y&3.04&100
         \\\hline
         GPT2-medium& 7(BitFit)&\N&\N&\Y&\N&\N&\N&3.42&0.062\\
         (heuristic)&8(LoRA)&\N&\N&\N&\Y&\Y&\N&3.20&0.111\\\hline\hline
         GPT2-large&1&\N&\Y&\N&\N&\N&\N&3.26&0.024 \\
         (ours)&2&\N&\Y&\N&\N&\Y&\N&3.32&0.072\\
         &3&\N&\Y&\Y&\N&\Y&\N&3.29&0.125\\
         &4&\N&\Y&\Y&\Y&\Y&\N&3.19&0.173\\
         &5&\N&\Y&\Y&\Y&\Y&\Y&3.03&8.645\\
         &FMT&\Y&\Y&\Y&\Y&\Y&\Y&2.96&100
         \\\hline
         GPT2-large& 7(BitFit)&\N&\N&\Y&\N&\N&\N&3.37&0.054\\
         (heuristic)&8(LoRA)&\N&\N&\N&\Y&\Y&\N&3.20&0.095\\\hline\hline
    \end{tabular}
}
    \label{tab:gpt2 PEFT}
\end{table}

\section{Related work}
\label{app:related work}

\paragraph{PEFT methods}
AdaPEFT can work compatibly with many PEFT methods, including those experimented in this work (LoRA and variants, linear probing, BitFit, and LayerNorm), those not experimented (e.g. prompt tuning \cite{lester2021power}, prefix tuning \cite{li2021prefix}, P-tuning \cite{liu2021p}, and adapter \cite{houlsby2019parameter}; see \cite{ding2023parameter,han2024parameter,wang2024parameter,prottasha2025peft} for a list of existing PEFT) and those yet-to-come. By taking more PEFT methods into consideration, we allow a larger search space for $\A$ and expect better performance from AdaPEFT, without noticeably increasing the computational cost if we scale the lazy updating accordingly.

\paragraph{Multi-task optimization}
Our multi-task formulation in \eqref{eq:main question0} and \eqref{eq:binary question MTO} is solved through $\epsilon$-constraint method. Another standard solution is the linear scalarization or weighted sum method: $\min_\A L_\A+a \frac{|\A|}{|\w|}$ for some tunable $a$. Note $\forall a>0$, without assuming convexity of the objective, this solution is guaranteed to be Pareto optimal. There are two potential challenges: (I) the scalarization problem is usually solved by gradient descent, but our subset selection problem is discrete and non-differentiable; (II) because each $a$ corresponds to one point on Pareto frontier, we need to try a number of $a$, whereas the $\epsilon$-constraint method only needs one sorting. Nevertheless, there may be multi-task optimization methods that can be directly applicable to our problems.

\paragraph{Hessian-informed loss and learning rate}
We have leveraged Hessian information to formulate our problems. Besides GeN \cite{bu2024automatic}, a line of work also proposed back-propagation-free ways to compute Hessian-informed learning rate \cite{fu2024qlabgrad,zhu2021automatic}, though the Hessian-informed loss reduction was not presented. It is also possible to use second-order back-propagation (such as Hessian-vector product or Hessian matrix instantiation) to compute the loss reduction we needed. However, this approach will be infeasible unless on very small models.

\end{document}